\def\informs{}

\documentclass[opre]{informs4_arxiv}

\RequirePackage{tgtermes}
\RequirePackage{newtxtext}
\RequirePackage{newtxmath}
\RequirePackage{bm}
\RequirePackage{endnotes}

\OneAndAHalfSpacedXII 

\usepackage{algorithm}
\usepackage{tikz}

\usepackage{natbib}
 \bibpunct[, ]{(}{)}{,}{a}{}{,}%
 \def\bibfont{\small}%
 \def\bibsep{\smallskipamount}%
 \def\bibhang{24pt}%
\usepackage{amsmath,amsfonts,cancel}
\usepackage{thmtools}
\usepackage[shortlabels]{enumitem}
\usepackage{mathtools}
\usepackage{dsfont}
\usepackage{tcolorbox}
\usepackage{booktabs}
\usepackage{bbm}
\usepackage{xspace}
\usepackage{subfigure}
\usepackage{tensor}
\usepackage{bm}
\usepackage{tikz}
\usepackage{placeins}
\usepackage{setspace}
\IfPackageLoadedTF{hyperref}{\usepackage{hyperref}}
{
\usepackage[colorlinks=true,breaklinks=true,bookmarks=true,urlcolor=magenta,citecolor=magenta,linkcolor=magenta,bookmarksopen=false,draft=false]{hyperref}}
\usepackage{algorithm}
\usepackage[noend]{algpseudocode}

\usepackage[capitalize]{cleveref}

\usepackage{natbib}

\usepackage{wrapfig}
\usepackage{xcolor}
\usetikzlibrary{matrix}
\usepackage{pgfplots}
\usepgfplotslibrary{fillbetween}
\pgfplotsset{compat=1.16,}
\usetikzlibrary{patterns}
\pgfplotsset{select coords between index/.style 2 args={
    x filter/.code={
        \ifnum\coordindex<#1\fi
        \ifnum\coordindex>#2\fi
    }
}}

\renewcommand{\P}{\mathcal{P}}  
\newcommand{\D}{\mathcal{D}}
\renewcommand{\H}{\mathcal{H}}

\newcommand{\mubar}{\overline{\mu}}

\newcommand{\basealgo}{\Pi}
\newcommand{\actionset}{\mathcal{A}}

\newcommand{\Exp}[1]{\mathbb E \left[ #1 \right]} 
\renewcommand{\Pr}{\mathbb{P}}
\newcommand{\OPT}{\textsc{Opt}}

\newcommand{\pfail}{\delta}
\newcommand{\B}{\mathcal{B}}

\newcommand{\eqdist}{\stackrel{d}{=}}

\newcommand{\point}{\mathbf{p}}
\newcommand{\pointi}[1]{\point^{(#1)}}

\newcommand{\pointset}{\mathcal{S}}
\newcommand{\budget}{B}

\newcommand{\points}{\vec{\point}}
\newcommand{\effortgeneric}{\beta}
\newcommand{\effort}{\vec{\beta}}
\newcommand{\efforti}[1]{\beta^{(#1)}}
\newcommand{\reg}{\mathcal{R}}
\newcommand{\effortset}{\mathcal{B}}

\newcommand{\regret}{\textsc{Reg}}
\newcommand{\A}{\mathcal{A}}
\newcommand{\action}{A} 
\newcommand{\actiont}[1]{\action_{#1}} 
\newcommand{\actionhist}{\actiongeneric^\mathcal{H}} 
\newcommand{\actionhistt}[1]{\actionhist_{#1}} 
\newcommand{\rewardhist}{R^\mathcal{H}} 
\newcommand{\rewardhistt}[1]{\rewardhist_{#1}} 
\newcommand{\actiongeneric}{a} 
\newcommand{\history}{\mathcal{H}^{\textit{hist}}}
\newcommand{\Hon}{\mathcal{H}^{\textit{used}}} 
\newcommand{\reward}{\mu}
\newcommand{\rewarddist}{\Re}
\newcommand{\obsreward}{R}
\newcommand{\obsrewardt}[1]{\obsreward_{#1}}
\newcommand{\dmax}{d_{\text{max}}}

\newcommand{\Optimal}{\textsf{Optimal}\xspace}
\newcommand{\Random}{\textsf{Random}\xspace}
\newcommand{\Full}{\textsf{Full Start}\xspace}
\newcommand{\Ignorant}{\textsf{Ignorant}\xspace}
\newcommand{\ArtReplay}{\textsf{Artificial Replay}\xspace}
\newcommand{\Regressor}{\textsf{Regressor}\xspace}
\newcommand{\CMABCRA}{CMAB-CRA\xspace}
\newcommand{\UCB}{\textsf{UCB}\xspace}
\renewcommand{\index}{\textsf{Index}\xspace}
\newcommand{\MonUCB}{\textsf{MonUCB}\xspace}
\newcommand{\adaalgo}{\textsf{AdaMonUCB}\xspace}
\newcommand{\iidata}{\textsf{IIData}\xspace}
\newcommand{\rel}{\doteq}

\newcommand{\psiopt}{\psi^*}
\newcommand{\aopt}{\actiongeneric^{\star}}

\newcommand{\abs}[1]{\left\lvert #1 \right\rvert}

\usepackage{multirow}

\mathchardef\mhyphen="2D 
\ifdefined \argmin
\else \DeclareMathOperator*{\argmin}{arg\,min}
\fi
\ifdefined \argmax
\else \DeclareMathOperator*{\argmax}{arg\,max}
\fi

\let\originalleft\left
\let\originalright\right
\renewcommand{\left}{\mathopen{}\mathclose\bgroup\originalleft}
\renewcommand{\right}{\aftergroup\egroup\originalright}

\renewcommand{\subparagraph}[1]{\noindent {\em #1.}~~}

\newenvironment{rproofof}[1]{ \ifdefined\informs \begin{proof}{\it Proof of #1.}
\else \begin{proof}[Proof of #1] \fi }{ \ifdefined\informs 
\Halmos \end{proof} \else \end{proof} \fi  }

\newenvironment{rproof}{
\noindent \ifdefined\informs \proof{\it Proof.}
\else \begin{proof} \fi }{ \ifdefined\informs 
\Halmos \endproof \else \end{proof} \fi}

\newcommand{\lily}[1]{\textcolor{blue}{[Lily: #1]}}

\definecolor{edits}{rgb}{1,0,0}

\usepackage{color-edits}
\addauthor{srs}{orange}
\addauthor{cy}{red}
\addauthor{sb}{teal}
\addauthor{lx}{cyan}

\usepackage{pifont}

\usepackage{amsmath,amsfonts,bm}

\def\eqref#1{equation~\ref{#1}}

\def\1{\bm{1}}

\DeclareMathAlphabet{\mathsfit}{\encodingdefault}{\sfdefault}{m}{sl}
\SetMathAlphabet{\mathsfit}{bold}{\encodingdefault}{\sfdefault}{bx}{n}

\EquationsNumberedThrough    

\TheoremsNumberedThrough     
\ECRepeatTheorems  %

\MANUSCRIPTNO{MOOR-0001-2024.00}

\begin{document}



\RUNAUTHOR{Banerjee et al.}

\RUNTITLE{Artificial Replay}

\TITLE{Artificial Replay: A Meta-Algorithm for Harnessing Historical Data in Bandits}

\ARTICLEAUTHORS{%
\AUTHOR{Siddhartha Banerjee}
\AFF{Cornell University, School of Operations Research and Information Engineering, \EMAIL{sbanerjee@cornell.edu}}

\AUTHOR{Sean R. Sinclair}
\AFF{Northwestern University, Department of Industrial Engineering and Management Sciences, \EMAIL{sean.sinclair@northwestern.edu}}

\AUTHOR{Milind Tambe}
\AFF{Harvard University, Department of Computer Science, \EMAIL{milind\_tambe@harvard.edu}}

\AUTHOR{Lily Xu}
\AFF{University of Oxford, Leverhulme Centre for Nature Recovery \& Department of Economics, \EMAIL{lily.x@columbia.edu}}

\AUTHOR{Christina Lee Yu}
\AFF{Cornell University, School of Operations Research and Information Engineering, \EMAIL{cleeyu@cornell.edu}}
} 

\ABSTRACT{Most real-world deployments of bandit algorithms exist somewhere in between the {\em offline} and {\em online} set-up, where some historical data is available upfront and additional data is collected dynamically online.  How best to incorporate historical data to ``warm start'' bandit algorithms is an open question: naively initializing reward estimates using all historical samples can suffer from spurious data and imbalanced data coverage, leading to data inefficiency (amount of historical data used) --- particularly for continuous action spaces. To address these challenges, we propose \ArtReplay, a meta-algorithm for incorporating historical data into \emph{any arbitrary base bandit algorithm}. We show that \ArtReplay uses only a fraction of the historical data compared to a full warm-start approach, while still achieving identical regret for base algorithms that satisfy \emph{independence of irrelevant data} (\iidata), a novel and broadly applicable property that we introduce. We complement these theoretical results with experiments on $K$-armed bandits and continuous combinatorial bandits, on which we model green security domains using real poaching data. Our results show the practical benefits of \ArtReplay for improving data efficiency, including for base algorithms that do not satisfy \iidata.
}%






\maketitle


\section{Introduction}

{Multi-armed bandits model decision-making settings with repeated choices between multiple actions (known as ``arms''), where the expected reward of each arm is unknown, and each arm pull yields a noisy sample from the true reward distribution.} Multi-armed bandits and their variants have been used effectively to model many real-world problems. Resulting algorithms have been applied to wireless networks~\citep{zuo2021combinatorial}, COVID testing regulations~\citep{bastani2021efficient}, and wildlife conservation to prevent poaching~\citep{xu2021dual}. Typical bandit algorithms are designed in one of two regimes: {\em offline} or {\em online}, depending on the assumptions on how data is collected.  In offline settings, the entire dataset is provided to the algorithm upfront, while in the online setting the algorithm starts from scratch and dynamically collects data online.  However, most practical deployments exist between these two extremes; in applying bandits to wildlife conservation, for example, we may have years of historical patrol records that could help learn poaching risk before initiating any bandit algorithm executed online to assign new patrol patterns.

A priori, it is not immediately clear how one should initialize an algorithm given historical data (especially when using more complex bandit algorithms).  As a warm-up, consider the well-studied UCB1 algorithm from \citet{lai1985asymptotically} for the $K$-armed bandit problem.  For a finite set of actions $\actiongeneric \in [K]$ with unknown mean reward $\mu(\actiongeneric)$, the algorithm maintains estimates $\mubar_t(\actiongeneric)$ for the mean reward of action~$\actiongeneric \in [K]$ and a count~$n_t(\actiongeneric)$ for the number of times the action~$\actiongeneric$ has been selected by the algorithm.  At time $t$ the algorithm then selects the action $A_t$ which maximizes the so-called upper confidence bound (UCB): \[\UCB_t(a) = \mubar_t( \actiongeneric) + O \left(1 / \sqrt{n_t(\actiongeneric)} \right) \ ,\]
where the $O(\cdot)$ drops polylogarithmic constants depending on the total number of timesteps~$T$.

A simple approach to incorporate historical data for this algorithm would be to initialize the mean $\mubar_t(\actiongeneric)$ and number of samples $n_t(\actiongeneric)$ at $t = 1$ based on the historical dataset, assuming it was collected by the bandit algorithm itself.  This approach uses that a ``sufficient statistic'' for UCB are the mean estimates and counts for each action.   We refer to this algorithm as \Full, which is practical and reasonable so long as the dataset is representative of the underlying model.  This algorithm has indeed been proposed for the simple $K$-armed bandit setting~\citep{shivaswamy2012multi}.

\subsection{Data Inefficiency of \Full}

While \Full is intuitive and simple, this naive approach can lead to excessive {\em data inefficiency} (the amount of historical data used by the algorithm), resulting in potentially unbounded computation and storage costs. We outline the challenges of \Full here, then in \cref{sec:benefits} we theoretically demonstrate that our approach can achieve arbitrarily better data efficiency than \Full, leading to significant reductions in computational and storage overhead.

In an extreme scenario, consider padding the historical data with samples from an arm $a$ whose mean reward $\mu(a)$ is lowest.  \Full will use the {\em entire} historical dataset regardless of its {\em inherent quality} at timestep $t = 1$.  This is even more surprising, since the celebrated result from \citet{lai1985asymptotically} shows that $O( \log(T) / \Delta(a)^2)$ samples are sufficient to discern that an action~$a$ is suboptimal (where $\Delta(a)$ is the gap between $\mu(a)$ and the reward of the optimal action). 	As the number of data points in the historical dataset from this suboptimal action tends towards infinity, \Full uses an unbounded amount of historical data, even though {\em only a fraction of the data} was necessary to appropriately rule out this action as being optimal. This inefficient use of data increases computational overhead and exacerbates the algorithm's data efficiency. Clearly, \Full is not an optimal approach for incorporating historical data, as it fails to selectively process only the most informative data points, leading to unnecessary computational and storage burdens.

\Full thus suffers from issues which we call {\em spurious data} and {\em imbalanced data coverage}, arising because the algorithm unnecessarily processes and stores data regardless of its quality.  These issues are particularly salient since many applications of bandit algorithms, such as online recommender systems, may have historical data with billions of past records; processing all data points would require exceptionally high upfront compute and storage costs.  Similarly, in the context of wildlife poaching, historical data takes on the form of past patrol data, which is guided by the knowledge of domain experts and geographic constraints.  As we will later see in \cref{sec:experiments}, the historical data is extremely geographically biased, so processing the entire historical dataset for regions with large amounts of historical data is costly and unnecessary.  Together, these challenges highlight a key aspect of incorporating historical data to warm-start any bandit algorithm: unless the historical data is collected by an optimal approach such as following a bandit algorithm, {\bf the value of information of the historical data may not be a direct function of its size.}  

Lastly, while we've so far discussed the stylized $K$-armed bandit setup as an example, many models of real-world systems require continuous actions over potentially combinatorial spaces with constraints.  For these more complex bandit settings, {it may not even be clear how to instantiate \Full, as the base algorithm may not clearly admit a low-dimensional ``sufficient statistic''}. Furthermore, the computation and storage costs increase dramatically with these more complex models, for example, when estimates of the underlying mean reward function are computed as a high-dimensional regression problem.  In such instances, the computational complexity for generating the estimate scales superlinearly with respect to the data efficiency (number of data points used).

Thus motivated, this paper seeks to answer the following questions:
\begin{itemize}
    \item {\em Is there a computationally efficient way to use only a subset of the historical data while maintaining the same regret guarantees of \Full?  Equivalently, is there a {\bf data efficient} way to incorporate historical data?
    \item Can we do so across different bandit set-ups, such as continuous or combinatorial bandits?
    \item How can we quantify the {\bf quality} or {\bf usefulness} of a given set of historical data?}
\end{itemize}
\noindent We will evaluate algorithms in terms of their performance (measured by regret), and data efficiency (measured by the number of historical datapoints accessed).  We emphasize that this later measure is correlated with the computation and storage costs for the algorithm.

\subsection{Our Contributions}

In our first contribution, we propose \ArtReplay, a meta-algorithm that modifies {\em any base bandit algorithm} to optimally harness historical data --- that is, using the minimal data required to achieve the highest possible performance. 
\ArtReplay improves data efficiency by using historical data \emph{as needed} --- specifically, only \emph{when recommended} by the base bandit algorithm.
Concretely, \ArtReplay uses the historical data as a replay buffer to artificially simulate online actions. When the base algorithm selects an action, we first check the historical data for any unused samples from the chosen action. If an unused sample exists, update the reward estimates using that historical data point and continue, \emph{without} advancing to the next timestep.  This allows the algorithm to obtain better estimates for the reward of that action without incurring additional regret. 
 Otherwise, take an online step by sampling from the environment (incurring regret if that action is non-optimal), and advance to the next timestep. We will demonstrate how this meta-algorithm can be applied to a wide set of practical models: standard $K$-armed bandits,
metric bandits with continuous actions, and models with semi-bandit feedback.


Given that \ArtReplay only uses a subset of the data, one might wonder if it incurs higher regret than \Full (a generalization of the algorithm described above for $K$-armed bandits). Surprisingly, in our second contribution we prove that under a widely applicable condition, the regret of \ArtReplay (as a random variable) is identical to that of \Full, while also guaranteeing significantly better data efficiency.
Specifically, we show a \emph{sample-path coupling}
between \ArtReplay and \Full with the same base algorithm, as long as the base algorithm satisfies a novel (and widely applicable) \emph{independence of irrelevant data} (\iidata) assumption. 
 We complement this result by also highlighting the regret improvement for \ArtReplay as a function of the used historical data. 
 Additionally, we prove that \ArtReplay can lead to arbitrary better computational complexity.  
These results highlight that \ArtReplay is a simple approach for incorporating historical data with identical regret to \Full and significantly better data efficiency.  To summarize, \ArtReplay Pareto-dominates \Full in terms of the two main metrics of interest: $(i)$~regret, $(ii)$~data efficiency.

Finally, we show the benefits of \ArtReplay by instantiating it for several classes of bandits and evaluating on real-world poaching data. To highlight the breadth of algorithms that satisfy the \iidata property, we show how standard UCB algorithms can be easily modified to be \iidata while still guaranteeing regret-optimality, such as for $K$-armed and continuous combinatorial bandits. While the results for $K$-armed bandits are fairly straightforward, the fact that these results extend to complex settings such as metric bandits with continuous actions or combinatorial bandits is not at all obvious. The surprising insight in our result is that a similarly simple approach for complex bandit settings can attain the same guarantees as full usage of the historical data with significant data efficiency gains.
Even for algorithms that do not satisfy \iidata, such as Thompson sampling and Information Directed Sampling (IDS), we demonstrate on real poaching data that \ArtReplay still achieves concrete gains in regret and data efficiency over a range of base algorithms.
We close with a case study of combinatorial bandit algorithms with continuous actions in the context of green security, using a novel adaptive discretization technique.

\subsection{Related Work}
\label{sec:related-work}

Multi-armed bandit problems and its sub-variants (including the finite-armed and \CMABCRA model to be discussed here) have a long history in the online learning and optimization literature.  We highlight the most closely related works below, but for more extensive references see~\cite{bubeck2012regret,aleks2019introduction,lattimore2020bandit}.

\paragraph{Multi-Armed Bandits.}\, The design and analysis of bandit algorithms have been considered under a wide range of models. \citet{lai1985asymptotically,auer2002finite} first studied the $K$-armed bandit, where the decision maker chooses between a finite set of $K$ actions at each timestep.  Numerous follow-up works have extended these approaches to handle continuous action spaces~\citep{kleinberg2019bandits} and combinatorial constraints~\citep{chen13a,xu2021dual}.
\citet{zuo2021combinatorial} propose a discrete model of the combinatorial multi-armed bandit that generalizes previous work \citep{lattimore2014optimal,lattimore2015linear,dagan2018better} to schedule a finite set of resources to maximize the expected number of jobs finished.  We extend their model to continuous actions, tailored to the green security setting from \citet{xu2021dual}.  Our work provides a framework to modify existing algorithms to efficiently incorporate historical data.  Moreover, we also propose and evaluate a novel algorithm to incorporate adaptive discretization for combinatorial multi-armed bandits for continuous resource allocation, extending the discrete model from~\citet{zuo2021combinatorial}.

\paragraph{Historical Data and Bandits.}\,  Several papers have started to investigate how to incorporate historical data into bandit algorithms, but current approaches in the literature only consider \Full. \citet{shivaswamy2012multi} employed this with UCB1 for $K$-armed bandits, then \citet{oetomo2021cutting} and \citet{wang17} used a similar approach in linear contextual bandits, regressing over the historical data to initialize the linear context vector. Similar techniques were used in models where there are pre-clustered arms, where the authors provide regret guarantees depending on the cluster quality of the fixed clusters~\citep{bouneffouf2019optimal}.  In contrast to these stylized approaches to a specific bandit model, our work provides a principled method for harnessing historical data across a variety of bandit models.  We show our meta-algorithm \ArtReplay can work with \emph{any} standard bandit framework and uses historical data \emph{as needed}, leading to improved data efficiency.  These techniques were also applied to Bayesian and frequentist linear contextual bandits where the linear feature vector is updated by standard regression over the historical data~\citep{oetomo2021cutting,wang17}.  The authors show empirically that these approaches perform better in early rounds, and applied the set-up to recommendation systems.  
A second line of work has considered warm-starting contextual bandit models with fully supervised historical data and online bandit interaction~\citep{swaminathan15a,zhang2019warm}.
In \citet{zuo2020observe} they consider augmented data collection schemes where the decision maker can ``pre sample'' some arms before decisions in the typical bandit set-up.  We further note that \citet{wagenmaker2022leveraging} considers how best to use offline data to minimize the number of online interactions necessary to learn a near-optimal policy in the linear RL setting.  Lastly, while not directly considering historical data, \citet{bayati2022speed} considers using a low-rank matrix model to reduce the exploration cost in two-sided bandit problems with many arms under short horizons.  We imagine that similar ideas can be incorporated to this work to study the impact of historical data on reducing the burn-in period for algorithms.

There is also recent work that has leveraged our model and algorithmic framework. \citet{cheung2024leveraging} considers leveraging offline data to facilitate online learning under distribution shift between offline data and online rewards.  They modify the UCB policy with a \Full approach that we study here but do not additionally consider optimally incorporating the historical data.  Additionally, \citet{agrawal2023optimal} builds on our work to consider best-arm identification in bandits with access to offline data. They instantiate \ArtReplay with a best-arm-identification base algorithm and show it yields a sub-optimal algorithm. They conjecture that a different base algorithm or tailored modifications to the \ArtReplay framework can be chosen to avoid this gap, due to the distinction between best-arm-identification and regret algorithms.

\paragraph{Bandit Algorithms for Green Security Domains.}
Green security focuses on allocating defender resources to conduct patrols across protected areas to prevent illegal logging, poaching, or overfishing \citep{fang2015security,plumptre2014efficiently}. 
These green security challenges have been addressed with game theoretic models \citep{yang2014adaptive,nguyen2016capture}; supervised machine learning \citep{kar2017cloudy,xu2020stay}; and multi-armed bandits, including restless \citep{qian2016restless}, recharging \citep{kleinberg2018recharging}, adversarial \citep{gholami2019don}, and combinatorial bandits \citep{xu2021dual}.
Related notions of the value of information have been studied in the context of ecological decision making to quantify how obtaining more information can help to better manage ecosystems \citep{canessa2015we}.  We conduct numerical simulations on this model in \cref{sec:experiments}.

\paragraph{Combinatorial Bandits for Resource Allocation.} The \CMABCRA model is a continuous extension of the combinatorial multi-armed bandit for discrete resource allocation (CMAB-DRA) problem studied in \cite{zuo2021combinatorial}. They propose two algorithms which both achieve logarithmic regret when the allocation space is finite or one-dimensional.  We extend their upper confidence bound algorithmic approach to consider both fixed and adaptive data-driven discretization of the continuous resource space, and additionally consider the impact of historical data in learning.  

\paragraph{Adaptive Discretization Algorithms.}
Discretization-based approaches to standard multi-armed bandits and reinforcement learning have been explored both heuristically and theoretically in different settings.  Adaptive discretization was first analyzed theoretically for the standard stochastic continuous multi-armed bandit model, where \cite{kleinberg2019bandits} developed an algorithm that achieves instance-dependent regret scaling with respect to the so-called ``zooming dimension'' of the action space.  This approach was later extended to contextual models in~\cite{slivkins2011contextual}.  In~\cite{elmachtoub2017practical} the authors improve on the practical performance and scalability by considering decision-tree instead of dyadic partitions of the action space.
Similar techniques have been applied to reinforcement learning, where again existing works have studied the theoretical challenges of designing discretization-based approaches with instance-specific regret guarantees~\citep{sinclair2021adaptive}, and heuristic performance under different tree structures~\citep{uther1998tree,pyeatt2001decision}.
However, none of these algorithms have been extensively studied within the concept of including historical data, or applied to the combinatorial bandit model, with the exception of~\cite{xu2021dual}.  Our work builds upon theirs through a novel method of incorporating historical data into an algorithm, and by additionally considering adaptive instead of fixed discretization.


\section{Preliminaries}
\label{sec:preliminary}

We start off by defining the general bandit model and assumptions we impose on the historical data, then in \cref{sec:bandit-models} we provide concrete examples with finite, continuous, and combinatorial actions, including the green-security domain we use in our simulations.

\subsection{General Stochastic Bandit Model} 
\label{sec:model_general}

We consider a stochastic bandit problem with a fixed feasible action set $\actionset$.  Let $\rewarddist : \actionset \rightarrow \Delta([0,1])$ be a collection of independent and  \emph{unknown} reward distributions over~$\actionset$. Our goal is to pick an action (also referred to as ``arm'') $\actiongeneric \in \actionset$ to maximize the expected reward~$\Exp{\rewarddist(\actiongeneric)}$, which we denote by $\reward(\actiongeneric)$. The optimal reward is $\OPT = \max_{\actiongeneric \in \actionset} \, \reward(\actiongeneric)$ under optimal action $\aopt = \argmax_{\actiongeneric \in \actionset} \mu(\actiongeneric)$.
For now, we do not impose any additional structure on~$\actionset$, but later in \cref{sec:bandit-models} we instantiate $\actionset$ to consider finite, continuous, and combinatorial constraints.

\paragraph{Historical Data.}  We assume that the algorithm designer has access to a historical dataset, which provides reward samples from a series of past actions. This historical dataset is comprised of $H_\actiongeneric \in \mathbb{Z}_{\geq 0}$ independent and identically distributed samples drawn from the reward~$\rewarddist(\actiongeneric)$ for each $\actiongeneric \in \actionset$.  Let $\history$ denote the set where each $\actiongeneric \in \actionset$ is repeated $H_\actiongeneric$ times, and (with slight abuse of notation) $\rewarddist(\history)$ the observed samples in the historical dataset.  Lastly, we use $H = \sum_{\actiongeneric \in \actionset} H_\actiongeneric$ to denote the total number of historical data points.

The assumption that $H_\actiongeneric$ is fixed and deterministic, as well as each sample is drawn according to $\rewarddist(\cdot)$, rules out several important tampering scenarios on the historical dataset.  For instance, consider an adversary who sorts the reward samples for each action $\actiongeneric$ to be in decreasing order, or an adversary that ``ignores'' samples smaller than a certain threshold value.  In both scenarios, the marginal reward distributions are biased for $\rewarddist(\cdot)$.
One could relax this assumption to assume that the historical data satisfies an {\em exchangeability} assumption, but we leave this for future work.  However, we note that as a follow-up to this work, \citet{cheung2024leveraging} consider the setting of a fixed bias in the sampling distribution of the historical dataset, and modify our \Full algorithms to account for this bias.

\paragraph{Online Structure.} Since the mean reward function~$\mu(\actiongeneric)$ is initially unknown, the algorithm must interact with the environment sequentially. At timestep $t \in [T]$, the decision maker picks an action $\actiont{t} \in \actionset$ according to their policy~$\pi$.  The environment then reveals a reward~$\obsrewardt{t}$ sampled from the distribution~$\rewarddist(\actiont{t})$. 
The optimal reward~$\OPT$ would be achieved using a policy with full knowledge of the true distribution.  We thus define \emph{regret} of a policy~$\pi$ as:
\begin{align}
\label{eq:regret}
    \regret(T, \pi, \history) = T \cdot \OPT - \textstyle\sum_{t=1}^T \reward(\actiont{t}) \ ,
\end{align}
where the dependence on $\history$ highlights that the selection of $\actiont{t}$ can depend on the historical dataset.  When taking expected regret, the expectation will be respect to the arm reward distribution for both {\em online} and {\em offline} data, but with $(H_\actiongeneric)_{\actiongeneric \in \actionset}$ fixed.  {As with standard bandit algorithms, our goal is to design algorithms with minimal regret, sublinear in both $T$ and $|\history|$.  However, departing from standard bandit literature, we will also measure the performance of algorithms by their data efficiency, measured in terms of the number of historical data points accessed by the algorithm.  We emphasize that data efficiency is correlated with the computation and storage complexity of the algorithms.  Hence, our primary goal will be to design algorithms with sublinear regret and data efficiency (both with respect to $T$ and $|\history|$).


\subsection{Finite, Continuous, and Combinatorial Action Spaces}
\label{sec:bandit-models}


The description in \cref{sec:model_general} {with an appropriate specification of} the feasible action set $\actionset$ {encompasses} most well-studied models in bandits.  Here we outline several examples, including the $K$-armed (\cref{sec:example_finite}) and continuous multi-armed bandit with combinatorial constraints (\cref{sec:def_cmab_cra}), which are used throughout the paper and in the simulation results of \cref{sec:experiments}. {With the exception of the $K$-armed and $K$-armed linear bandit models,} the question of how best to incorporate {historical} data {to warm start bandit algorithms} has not been previously considered in the literature {for any of the other bandits models we consider here}. {We present a single meta-algorithm that can be used} to naturally incorporate historical data across each of these different bandit models. In \cref{sec:algorithms} we {instantiate the algorithm for these models and show that this approach leads to} regret-optimal algorithms.

\subsubsection{$K$-armed Bandit}
\label{sec:example_finite}

The finite-armed bandit model can be viewed in this framework by considering $\actionset = [K] = \{1, \ldots, K\}$.  This recovers the classical model from \citet{lai1985asymptotically,auer2002finite}.  \citet{shivaswamy2012multi} started the study of incorporating historical data into these models through the use of the standard UCB algorithm and the \Full approach. We will later extend these ideas to cover a much wider class of base bandit algorithms, and additionally improve data efficiency by only using the historical data as needed.

\subsubsection{$K$-armed Linear Bandit}

The finite-armed linear bandit model can be viewed in this framework by again considering $\actionset = [K] = \{1, \ldots, K\}$.  However, here it is additionally assumed that each action $\actiongeneric$ has an associated known feature vector $\phi_\actiongeneric \in [0,1]^d$ and the mean reward $\mu(\actiongeneric)$ satisfies $\mu(\actiongeneric) = \phi_\actiongeneric^\top \theta$ where $\theta \in [0,1]^d$ is an unknown latent parameter~\citep{bastani2020online}.  Incorporating historical data in this model was first investigated in \citet{oetomo2021cutting,wang17} where the historical data is used through regression to estimate the latent parameter $\theta$.  We note that this is again an instantiation of the \Full approach to the algorithm developed in \citet{bastani2020online}.


\subsubsection{Metric Bandits}
\label{sec:def_metric_bandit}
The metric or continuous bandit model from \citet{kleinberg2019bandits} can be viewed in this framework where we assume that the action set $\actionset$ is a compact metric space with metric $d$.  The mean reward function $\mu(a)$ is then assumed to be Lipschitz with the metric $d$, in that there exists a constant $L$ such that:
\[
\abs{\mu(a) - \mu(a')} \leq L d(a,a') \quad\quad \forall a, a' \in \actionset \ .
\]

\subsubsection{Combinatorial Finite-Armed Bandits}
\label{sec:def_cmab}
Here we consider a central planner who has access to a finite set of $[K]$ actions with unknown mean $\mu(\actiongeneric)_{\actiongeneric \in [K]}$ and over each round is tasked with picking at most $B \leq K$ of them (see~\citet{chen13a} for a more general model).  We denote the feasible action space as the set of allocation vectors $\effort$ as follows:
\begin{align}
    \actionset = \left\{ \effort \in \{0,1\}^{K} \mid \sum\nolimits_{\actiongeneric \in [K]} \efforti{\actiongeneric} \leq B\right\} \ .
\end{align}
It is further assumed that the mean reward over $\actionset$ is additive, in that for any $\effort \in \actionset$ we have:
$\mu(\effort) = \sum_{\actiongeneric \in [K]} \efforti{\actiongeneric} \mu(\actiongeneric),$
and the algorithm receives feedback of the form of $\obsrewardt{} \sim \rewarddist(\actiongeneric)$ for each $\actiongeneric \in [K]$ with $\efforti{\actiongeneric} = 1$.
The goal of the algorithm is then to learn the top $B$ actions in order to maximize their cumulative mean reward.  Lastly, we consider the historical data $\history$ to also be decomposed, in that each element is a particular $(\actiongeneric, \obsrewardt{})$ pair with $\obsrewardt{} \sim \rewarddist(\actiongeneric)$.

\subsubsection{Combinatorial Multi-Armed Bandit for Continuous Resource Allocation (\CMABCRA)}
\label{sec:def_cmab_cra}

This model combines the metric and Lipschitz structure of \cref{sec:def_metric_bandit} with the combinatorial structure of \cref{sec:def_cmab}.  We consider a central planner who has access to a metric space~$\pointset$ of resources with metric~$d_\pointset$.  They are tasked with splitting a total amount of $\budget$ divisible budget across $N$~different resources within~$\pointset$.  For example, in wildlife conservation, the space~$\pointset$ can be considered as the protected area of a park, and the allocation budget corresponds to divisible effort, or the proportion of rangers allocated to patrol in the chosen area.  We denote the feasible space of budget allocations as~$\effortset$ and define the feasible action space as follows:
\begin{align}
\label{eq:action_set}
    \actionset = \left\{ (\points, \effort) \in \pointset^N \times \effortset^N \bigm\vert \sum_{i=1}^N \efforti{i} \leq \budget, \quad d_\pointset(\pointi{i},\pointi{j}) \geq \epsilon \quad \forall i \neq j \right\} \ .
\end{align}
Note that we require the chosen action to satisfy the budget constraint (i.e. $\sum_i \efforti{i} \leq \budget$), and that the chosen resources are distinct (e.g., $\epsilon$-away from each other according to $d_\pointset$).

We further assume that the reward distribution $\rewarddist$ decomposes additively over $\actionset$ as follows.  Indeed, overload notation and let $\rewarddist : \pointset \times \effortset \rightarrow \Delta([0,1])$ be the \emph{unknown} reward distribution over the resource and allocation space. The goal of the algorithm is to pick an action $\action = (\points, \effort) \in \actionset$ in a way that maximizes $\sum_{i=1}^N \Exp{\rewarddist(\pointi{i}, \efforti{i})}$, the expected total mean reward accumulated from the resources subject to the budget constraints. Denoting $\Exp{\rewarddist(\point, \effortgeneric)}$ as $\reward(\point, \effortgeneric)$, the optimization problem to determine the optimal $\actiongeneric = (\points, \effort) \in \actionset$ is formulated below:
\begin{align}
    \max_{\points, \effort} & \, \sum_{i=1}^N \reward(\pointi{i}, \efforti{i}) \label{eq:cmab_opt} \\
    \text{s.t. } & \sum_i \efforti{i} \leq B \nonumber\\
    & d_\pointset(\pointi{i}, \pointi{j}) \geq \epsilon \quad \forall i \neq j \ . \nonumber
\end{align}
\noindent Lastly, we consider the historical data $\history$ to also be decomposed, in that each element is a particular $(\point, \effortgeneric, \obsrewardt{})$ pair with $\obsrewardt{} \sim \rewarddist(\point, \effortgeneric)$.

In this domain, we make two assumptions on the underlying problem. The first combines the metric bandit assumption of \cref{sec:def_metric_bandit}, where we assume that the resource space~$\pointset$ is a compact metric space with metric $d_\pointset$ and that the true underlying reward function~$\reward(\point, \effortgeneric)$ is Lipschitz with respect to~$d_\pointset$.  The next assumption assumes access to an oracle for solving optimization problems of the form of \cref{eq:cmab_opt} for arbitrary choice of reward functions $r(\point, \effortgeneric)$.  We can relax this assumption to instead assume that there exists a randomized approximation oracle by appropriately shifting the regret benchmark.  However, in \cref{sec:experiments} we run experiments with exact solution oracles and omit this discussion from this work.

\paragraph{Mapping to Green Security Domains} The \CMABCRA model can be used to specify green security domains from \citet{xu2021dual}.  $\pointset$~is used to represent a protected area, which is a geographic space designated for wildlife conservation. $\effortset$~is the discrete set of potential patrol efforts to allocate, such as the number of ranger hours per week, with the total budget~$B$ being $40$~hours.  The reward function $\reward(\point, \effortgeneric)$ then models the probability of observing a snare at location $\point$ at effort level $\effortgeneric$.

In the simplest case, $\pointset$ is a two-dimensional space where every point corresponds to a GPS coordinate in the park. For more precision, we could consider planning over not just the physical space, but rather the feature space. In that case, $\pointset$ is a $k$-dimensional space over $k$ features in the park that influence green security threats (such as land cover, elevation, and distance to rivers that influence poaching). For purposes of exposition, in this paper we consider the former case, where $\pointset$ directly corresponds to GPS coordinates.  This formulation generalizes \citet{xu2021dual} to a more realistic continuous space model of the landscape, enabling greater geographic precision in resource allocation than the artificial fixed discretization that was considered in prior work (consisting of $1 \times 1$ sq.~km regions of the park).

\section{\ArtReplay for Harnessing Historical Data}
\label{sec:online_wrapper}

In this section we propose \ArtReplay, a simple approach to incorporate historical data into bandit algorithms.  We start off with a case study of UCB for $K$-armed bandits (\cref{sec:example_finite}) for which the instantiation is more straightforward.  

However, for algorithms in complex settings (such as the others discussed in \cref{sec:bandit-models}) it is not obvious how one should initialize the algorithm given historical data.  We first discuss \Full, which initializes an algorithm by using the entire historical data upfront.  We later develop \ArtReplay, a meta-algorithm to efficiently incorporate historical data which can be integrated with any base algorithm.  We highlight how \ArtReplay depends on an {\em action-similarity} relationship, which captures the notion of whether data from a given action can be used to inform the reward distribution of other actions. 

\subsection{Case study: \MonUCB for $K$-armed Bandits}
\label{sec:ar_case_study_mon_ucb}

\begin{algorithm}[tb]
   \caption{Monotone UCB (\MonUCB)}\label{alg:mon_ucb}
\begin{algorithmic}[1]
   \State Initialize $n_1(\actiongeneric) = 0$, $\mubar_1(\actiongeneric) = 1$, and $\UCB_1(\actiongeneric) = 1$ for each $\actiongeneric \in [K]$
   \For{$t = \{1, 2, \ldots\}$}
   \State Let $\actiont{t} = \argmax_{\actiongeneric \in [K]} \UCB_t(\actiongeneric)$
   \State Receive reward~$\obsrewardt{t}$ sampled from $\rewarddist(\actiont{t})$
   \For{all $\actiongeneric \neq \actiont{t}$}
   \State $n_{t+1}(\actiongeneric) = n_t(\actiongeneric)$, $\mubar_{t+1}(\actiongeneric) = \mubar_t(\actiongeneric)$, $\UCB_{t+1}(\actiongeneric) = \UCB_t(\actiongeneric)$
   \EndFor
   \State $n_{t+1}(\actiont{t}) = n_t(\actiont{t}) + 1$
   \State $\mubar_{t+1}(\actiont{t}) = (n_t(\actiont{t}) \mubar_t(\actiont{t}) + \obsrewardt{t}) / n_{t+1}(\actiont{t})$
   \State $\UCB_{t+1}(\actiont{t}) = \min\{\UCB_t(\actiont{t}), \mubar_{t+1}(\actiont{t}) + \sqrt{2 \log(T) / n_{t+1}(\actiont{t})}\}$
   \EndFor
\end{algorithmic}
\end{algorithm}

To first address the problem of incorporating historical data into a bandit algorithm, we consider a modification of the standard UCB1 algorithm from \citet{auer2002finite}.  We first ignore the historical data and propose \MonUCB, which is identical to standard UCB1 except for two small modifications. First, we adjust the confidence terms to depend on $\log(T)$ instead of $\log(t)$. Second, we force the $\UCB$ estimates to be monotone decreasing over timesteps~$t$. These modifications \emph{have no effect on the regret guarantee}. Under the ``good event'' analysis, if $\UCB_{t}(\actiongeneric) \geq \reward(\actiongeneric)$ with high probability, then the condition still holds at time~$t+1$, even after observing a new data point.  We then modify the \MonUCB algorithm to use either the full historical data upfront, or adaptively using the \ArtReplay framework.

\MonUCB tracks the same values as UCB1: \textit{(i)}~$\mubar_t(\actiongeneric)$ for the estimated mean reward of~$\actiongeneric$, \textit{(ii)}~$n_t(\actiongeneric)$ for the number of times~$\actiongeneric$ has been selected by the algorithm, and \textit{(iii)}~$\UCB_t(\actiongeneric)$ for an upper confidence bound estimate for the mean reward of~$\actiongeneric$. 
At every timestep~$t$, \MonUCB picks the action~$\actiont{t}$ which maximizes $\UCB_t(\actiongeneric)$ (breaking ties deterministically).  After observing~$\obsreward_t \sim \rewarddist(\actiont{t})$, we update our counts $n_{t+1}(a)$ and estimates $\mubar_{t+1}(a)$---but, critically, we only update $\UCB_{t+1}(\actiongeneric)$ if the new UCB estimate is less than the prior UCB estimate $\UCB_{t}(\actiongeneric)$. 
The full pseudocode is given in Algorithm~\ref{alg:mon_ucb}.  

We start by highlighting the fact that these changes to UCB1 {\em do not} affect the regret.  Indeed, \MonUCB achieves the same instance-dependent regret bound as the standard UCB1 algorithm. 

\begin{theorem}
\label{thm:mon_ucb_regret}
The \MonUCB base algorithm has for $\Delta(\actiongeneric) = \max_{\actiongeneric^\prime} \mu(\actiongeneric^\prime) - \mu(\actiongeneric)$: 
\begin{align*}
    & \Exp{\regret \left( T, \pi^{\MonUCB} \right)} = O\left( \sum_{\actiongeneric \neq \aopt} \log(T) / \Delta(\actiongeneric) \right) \ . 
\end{align*}
\end{theorem}

\begin{rproof} See \cref{sec:k_armed_proofs}.
\end{rproof}
\noindent We note that the monotonicity modification can be applied to more general $\psi$-UCB algorithms from~\citet{bubeck2012regret} (but for expository purposes we defer this discussion to \cref{sec:ucb_algorithms_appendix}).

\paragraph{Incorporating all historical data upfront.} As written, \MonUCB does not directly incorporate the historical data $\history$.  We first consider a straightforward modification, resulting in a policy that we denote as $\pi^{\Full(\MonUCB)}$, which utilizes the entire historical data upfront.  The only modification for $\Full(\MonUCB)$ is that at $t = 1$ (Line 1 of \cref{alg:mon_ucb}) we initialize:
\textit{(i)}~$\mubar_1(\actiongeneric)$ for the estimated mean reward of action~$\actiongeneric$ taken over the $H_\actiongeneric$ samples in $\history$, \textit{(ii)}~$n_1(\actiongeneric)$ for the number of times the action~$\actiongeneric$ was selected in $\history$, i.e. $H_\actiongeneric$.  After the first round, the remaining steps are identical, where new data is used to update the {\em already pre-initialized} $\mubar_1(\actiongeneric)$ and $n_1(\actiongeneric)$ values.  Here we see the importance of the first modification of \MonUCB: replacing $\log(t)$ with $\log(T)$ in the confidence radius ensures that the $\UCB_t(\actiongeneric)$ values are properly calibrated with a union bound~\citep{shivaswamy2012multi}.

Unfortunately, $\pi^{\Full(\MonUCB)}$ has several drawbacks as hinted at in the introduction.  Consider a sub-optimal action $\actiongeneric \in [K]$.  Then by padding the historical dataset with additional samples of $\actiongeneric$, $\Full(\MonUCB)$ will read in the entire dataset for $H_\actiongeneric$, {\em even though} the celebrated result by \citet{lai1985asymptotically} shows that $O(\log(T) / \Delta(\actiongeneric)^2)$ samples are sufficient to discern that $\actiongeneric$ is suboptimal.  This simple example highlights the data inefficiency of the \Full approach.

\paragraph{A more natural and adaptive way to incorporate historical data.} Based on this drawback to $\Full(\MonUCB)$, the obvious question is how can we adaptively incorporate the historical data to ensure that at most $O( \log(T) / \Delta(\actiongeneric)^2)$ samples from each suboptimal action are used?  The issue with $\Full(\MonUCB)$ is that the historical data is used independent of its quality.  A better strategy would be to only use the historical data for an action $\actiongeneric$ when \MonUCB suggests to play that action.  This automatically ensures that historical data for high-performing actions is used, since by design, \MonUCB plays near-optimal actions.  We formalize this through the following steps.  Let $\H_t$ be the set of {\em online data} observed by the algorithm, and $\Hon_t$ as the subset of $\history$ containing the used historical data, where initially $\Hon_1 = \emptyset$. We also abuse notation and let $\mubar(\H, \actiongeneric)$ and $n(\H, \actiongeneric)$ to be the empirical mean reward and number of appearances for any $\actiongeneric$ in $\actionset$.

We construct a policy, which we denote as $\pi^{\ArtReplay(\MonUCB)}$ as follows.  Initially, at $t = 1$ we have that $\mubar_t(\actiongeneric) = 1$, $n_t(\actiongeneric) = 0$, and $\UCB_t(\actiongeneric) = 1$.  Over rounds $t$, we let $\tilde{\actiont{t}} = \argmax_{\actiongeneric \in [K]} \UCB_t(a)$ be the {\em proposed action}.  In order to only use historical data from {\em good} actions, since the algorithm proposes to play $\tilde{\actiont{t}}$ we first check if the set of {\em unused} historical data $\history \setminus \Hon_t$ contains any {\em additional samples} from~$\tilde{\actiont{t}}$.  If so, we update $\mubar_t(\tilde{\actiont{t}})$, $n_t(\tilde{\actiont{t}})$, and $\UCB_t(\tilde{\actiont{t}})$ with this additional sample, otherwise we play $\tilde{\actiont{t}}$ and continue to the next timestep.
\begin{itemize}
    \item \textbf{No historical data available}: If $\history \setminus \Hon_t$ does not contain any additional samples from $\tilde{\actiont{t}}$, then the \emph{selected action} is $\actiont{t} = \tilde{\actiont{t}}.$ We keep $\Hon_{t+1} = \Hon_t$, set $\H_{t+1} = \H_t \cup \{(\actiont{t}, \rewarddist(\actiont{t})\}$, and advance to timestep~$t+1$ by updating:
    \begin{align*}
    n_{t+1}(\actiont{t}) = n \left( \Hon_{t+1} \cup \H_{t+1}, \actiont{t} \right), \quad\quad
    \mubar_{t+1}(\actiont{t}) = \mubar \left( \Hon_{t+1} \cup \H_{t+1}, \actiont{t} \right), \\
    \UCB_{t+1}(\actiont{t}) = \min \left\{ \UCB_t(\actiont{t}), \mubar_{t+1}(\actiont{t}) + \sqrt{2 \log(T) / n_{t+1}(\actiont{t})} \right\} \ .
    \end{align*}
    \item \textbf{Historical data available}: If $\tilde{\actiont{t}}$ is contained in $\history \setminus \Hon_t$, then we add that data point to $\Hon_t = \Hon_t \cup \{(\tilde{\actiont{t}}, \rewarddist(\tilde{\actiont{t}})\}$.  We update the estimates as
    \begin{align*}
    n_{t}(\actiont{t}) = n \left( \Hon_{t} \cup \H_{t}, \actiont{t} \right), \quad\quad
    \mubar_{t}(\actiont{t}) = \mubar \left( \Hon_{t} \cup \H_{t}, \actiont{t} \right), \\
    \UCB_{t}(\actiont{t}) = \min \left\{ \UCB_t(\actiont{t}), \mubar_{t}( \actiont{t}) + \sqrt{2 \log(T) / n_{t}(\actiont{t})} \right\} \ ,
    \end{align*}
    and repeat by picking another \emph{proposed action} $\tilde{\actiont{t}} = \argmax_{\actiongeneric \in [K]} \UCB_t(\actiongeneric)$.  We remain at timestep~$t$.
\end{itemize}
By design, $\ArtReplay(\MonUCB)$ only incorporates historical data when $\MonUCB$ itself decides it wants to play an action.  We will later see in \cref{thm:comp_gains} that this avoids the issues of $\Full(\MonUCB)$ and only reads in a subset of the historical data necessary to ensure that an action is sub-optimal.  
While this simple idea seems {\em almost} obvious within the context of $K$-armed bandits, the appropriate extension to general bandit models can be quite complex. {For example, in a continuous bandit setting as described in \cref{sec:def_metric_bandit}, one wouldn't expect the historical data to contain multiple data points for any single given action, and it may even be likely that the dataset contains no data at all for any of the discrete set of actions recommended by any given base bandit algorithm. However, this does not mean that the historical data contains no relevant information about the set of actions, as data from nearby actions provides weak information about the reward due to the Lipschitz assumption on the reward function. Formalizing the appropriate way to utilize this historical data will require care.} Next, we provide a principled method for incorporating historical data across a variety of bandit scenarios (including those in \cref{sec:bandit-models}).

\subsection{Incorporating historical data: \Ignorant and \Full approaches}

In order to generalize the previous discussion for $K$-armed bandits, we model algorithms for online stochastic bandits as a function mapping histories (collections of observed $(\actiongeneric, \obsreward)$ pairs) to a distribution over actions in $\actionset$.  We let $\basealgo : \D \rightarrow \Delta(\actionset)$ denote an arbitrary {\bf base algorithm} where $\D$ denotes the collection of possible unordered histories (i.e., $\D = \cup_{i \geq 0} (\actionset \times \mathbb{R}_+)$).  This definition ensures that the algorithm is designed based on a sufficient statistic of the data (for example, the empirical mean reward estimate for the case of \MonUCB).  It also importantly assumes that the algorithm's action selection is identical regardless of the {\em order of observations}.  This also requires that the action selection does not depend on $t$.  Again, in \MonUCB, the order of observations does not impact the final decision, since the decision is determined based on the empirical average reward per action and number of times each action was played; those two aggregate values are sufficient statistics.  While this restricts the set of algorithms considered, in \cref{sec:algorithms} we will provide algorithms that satisfy the required assumptions across the bandit models discussed in \cref{sec:bandit-models}.


With these definitions in hand, we now define the \textbf{ignorant} algorithm, which ignores the historical data, and the \textbf{full warm start} algorithm, which includes all the historical data upfront.

\begin{definition}[Ignorant algorithm]
The \textbf{ignorant} algorithm (\Ignorant) obtained by a base algorithm~$\basealgo$ that does not incorporate historical data takes actions according to the policy:
\[
\pi^{\Ignorant(\basealgo)}_t = \basealgo(\mathcal{H}_t)
\]
where $\H_t \in \D$ is the online data observed by timestep~$t$.
\end{definition}

\begin{definition}[Full warm start]
The \textbf{full warm start} algorithm (\Full), which uses the full historical data upfront, takes actions according to the policy:
\[
\pi^{\Full(\basealgo)}_t = \basealgo(\history \cup \H_t).
\]
\end{definition}

Here we finally see a formal definition of the \Full approach, which incorporates the full historical data upfront, and recovers the example highlighted in \cref{sec:ar_case_study_mon_ucb} for the $\Full(\MonUCB)$ algorithm.

\subsection{Equivalence Classes over Actions}
\label{sec:information_relation}

As we presented in \cref{sec:ar_case_study_mon_ucb}, a key step of \ArtReplay applied to \MonUCB is to check whether the historical data contains unused samples for the current proposed action $\tilde{\actiont{t}}$.  While this is simple in the $K$-armed bandit setup, with continuous or combinatorial actions it is unlikely that each individual action is multiply represented within the historical dataset.  However, ``similar'' actions can still be used in order to infer additional information on the reward distribution for $\tilde{\actiont{t}}$. We thus introduce the notion of {\em equivalence relations} over actions, which is key to establishing \ArtReplay for more general base bandit algorithms.

In \cref{sec:ar_case_study_mon_ucb} the equivalence is clear by taking $\actiongeneric \rel \actiongeneric'$ if and only if $\actiongeneric = \actiongeneric'$.  However, in more complicated bandit models, such as continuous or combinatorial bandits (\cref{sec:def_metric_bandit,sec:def_cmab,sec:def_cmab_cra}), we will need to use a coarser relation.  For example, consider the metric bandit model of \cref{sec:def_metric_bandit}.  Under the Lipschitz assumption on the mean reward we have that:
\[
\abs{\mu(\actiongeneric) - \mu(\actiongeneric')} \leq Ld(\actiongeneric, \actiongeneric'),
\]
where $d$ is the metric over $\actionset$. Then incorporating historical data from actions $\actiongeneric$ where $d(\actiongeneric, \tilde{\actiont{t}}) \leq \epsilon$ would allow us to infer some information about $\tilde{\actiont{t}}$, which can be exploited to learn the optimal action faster. We could safely ignore historical data from actions $\actiongeneric$ where $d(\actiongeneric, \tilde{\actiont{t}})$ is large, which would maintain that only a necessary subset of the historical data is used.  To formalize this idea, we assume that the action set $\actionset$ is endowed with a relation as defined below.
\begin{definition}
\label{def:action_relation}
We define the {\bf information relation} over $\actionset$ as a set of relations indexed over $\H \in \D$ for any set of histories, which we denote as $\rel_\H$.  
\end{definition}
\noindent When measuring the computational complexity of the algorithms, we will assume that checking the relation $\rel_\H$ can be done in constant time~\footnote{This will be possible for the algorithms we consider in this work, but in practice checking $\rel_\H$ will be algorithm dependent.}.

The action set $\actionset$ is endowed with several vacuous relations.  For example, the trivial relation is one where $\actiongeneric \rel \actiongeneric'$ for any $\actiongeneric, \actiongeneric' \in \actionset$.  
Another standard relation would be to assume that all actions are only related to themselves (i.e. $\actiongeneric \rel_\H \actiongeneric'$ if and only if $\actiongeneric = \actiongeneric'$).  This captures the model from \cref{sec:ar_case_study_mon_ucb}, but does not incorporate the Lipschitz assumption from the metric bandits in \cref{sec:def_metric_bandit}.  We also allow the information relation to depend on the history $\H$.  This need arises since the choice of relation should be algorithm (and potentially {\em data}) dependent, as we will later see in \cref{sec:algorithms}.  When there is no dependence on $\H$ we will abuse notation and omit it.

While we will give concrete instantiations of base algorithms alongside their information relation later in \cref{sec:algorithms}, we briefly give some relations under different action spaces as follows.

\paragraph{Continuous Actions.} Typical algorithms for continuous action spaces work by constructing a discretization over $\actionset$ according to a fixed bandwidth parameter $\gamma$, and running a standard UCB algorithm over the (now-discrete) approximate action set~\citep{kleinberg2019bandits}.  
A natural information relation would be to set $\actiongeneric \rel \actiongeneric'$ whenever $d(\actiongeneric, \actiongeneric') \leq \gamma$.  However, an adaptive discretization algorithm requires a history-dependent information relation that encodes the discretization at the time of the proposed action.  We will see this in \cref{sec:algorithms}.

\paragraph{Semi-Bandit Feedback.} In \cref{sec:def_cmab,sec:def_cmab_cra} we introduced bandit models which observe semi-bandit feedback.  Suppose that actions can be written as $\actiongeneric = (\actiongeneric_1, \ldots, \actiongeneric_N)$ and the reward function decomposes additively with independent rewards for each subcomponent~$\actiongeneric_i$ (sometimes referred to as a ``subarm'' in combinatorial bandits).  A natural information relation would be to set $\actiongeneric = (\actiongeneric_1, \ldots, \actiongeneric_N) \sim \actiongeneric' = (\actiongeneric'_1, \ldots, \actiongeneric'_N)$ whenever there exists an index $j$ such that $\actiongeneric_j \rel \actiongeneric'_j$.  


\begin{algorithm}[tb]
   \caption{\ArtReplay} \label{alg:online-wrapper}
\begin{algorithmic}[1]
   \State {\bfseries Input:} Historical dataset $\history$, base algorithm~$\basealgo$, relation over $\actionset$ $\rel_\H$
   \State Initialize set of used historical data points~$\Hon_1 = \emptyset$ and set of online data $\H_1 = \emptyset$
   \For{$t = \{1, 2, \ldots\}$}
   \State Initialize $\texttt{flag}$ to be $\texttt{True}$
   \While{\texttt{flag} is \texttt{True}}
   \State Pick action $\tilde{\actiont{t}} \sim \basealgo(\Hon_t \cup \H_t)$
   \State \textit{$\triangleright$ If an unused {and relevant} sample {does not} exist in the historical dataset}
   \If{$(\history \setminus \Hon_t) \cap \{\actiongeneric \in \actionset : \actiongeneric \rel_{\Hon_t \cup \H_t} \tilde{\actiont{t}}\} = \emptyset$}
   \State Set online action $\actiont{t} = \tilde{\actiont{t}}$
   \State Execute $\actiont{t}$ and observe reward $\obsrewardt{t} \sim \rewarddist(\actiont{t})$ 
   \State Update $\H_{t+1} = \H_t \cup \{(\actiont{t}, \obsrewardt{t})\}$ and $\Hon_{t+1} = \Hon_t$
   \State Update $\texttt{flag}$ to be $\texttt{False}$ 
   \Else
  \State Update $\Hon_t$ to include a sample for any $\actiongeneric$ in $(\history \setminus \Hon_t) \cap \{\actiongeneric \in \actionset : \actiongeneric \rel_{\Hon_t \cup \H_t} \tilde{\actiont{t}}\}$ 
  \EndIf
   \EndWhile
   \EndFor
\end{algorithmic}
\end{algorithm}

\subsection{\ArtReplay}
The \ArtReplay meta-algorithm takes as input an equivalence relation $\rel$ over $\actionset$ and \emph{simulates} the base algorithm using historical data as a replay buffer, resulting in a policy which we denote $\pi^{\ArtReplay(\basealgo)}$. The algorithm keeps track of a set~$\Hon_t$ of historical data points used by the start of time~$t$.  Initially, $\Hon_1 = \emptyset$.  
For an arbitrary base algorithm $\basealgo$, let $\tilde{\actiont{t}} \sim \basealgo(\Hon_t \cup \H_t)$ be the {\em proposed action}. There are two possible cases: whether or not the current set of unused historical data points $\history \setminus \Hon_t$ contains any {\em additional samples} from actions related to the proposed action~$\tilde{\actiont{t}}$. Specifically, we test if $(\history \setminus \Hon_t) \cap \{\actiongeneric \in \actionset \mid \actiongeneric ~\rel_{\Hon_t \cup \H_t} \tilde{\actiont{t}}\} = \emptyset$.  If so, the historical data is {\em uninformative} of the current proposed action, otherwise some action in the historical data can be used to update the information about $\tilde{\actiont{t}}$.
\begin{itemize}
    \item \textbf{No historical data available}: If there does not exist an $\actiongeneric \rel_{\Hon_t \cup \H_t} \tilde{\actiont{t}}$ contained in $\history \setminus \Hon_t$, then the \emph{selected action} is: $$\pi^{\ArtReplay(\basealgo)}_t(\Hon_t \cup \H_t) = \actiont{t} = \tilde{\actiont{t}}.$$ We also keep $\Hon_{t+1} = \Hon_t$ and advance to timestep~$t+1$.
    \item \textbf{Historical data available}: If there exists an $\actiongeneric \rel_{\Hon_t \cup \H_t} \tilde{\actiont{t}}$ contained in $\history \setminus \Hon_t$, add that data point to $\Hon_t = \Hon_t \cup \{(\actiongeneric, \rewarddist(\actiongeneric))\}$ and repeat by picking another \emph{proposed action}:
    \[
    \tilde{\actiont{t}} = \pi^{\ArtReplay(\basealgo)}_t(\Hon_t \cup \H_t).
    \]
    We remain at time~$t$.
\end{itemize}
Each step is guaranteed to result in a {\em selected action} $A_t$ since the historical data is finite.  We provide pseudocode in \cref{alg:online-wrapper} and use $\pi^{\ArtReplay(\basealgo)}$ to denote the resulting policy, omitting the dependence on the information relation $\rel$.

An important step of \ArtReplay is to check whether the unused historical data contains additional samples relevant to the proposed action $\tilde{\actiont{t}}$.
Since this sampling aligns with the base algorithm $\basealgo$, it inherently focuses on high-performing actions. Under the trivial relation where $\actiongeneric \rel \actiongeneric’$ for all $\actiongeneric, \actiongeneric’ \in \actionset^2$, the \ArtReplay algorithm becomes equivalent to \Full, reading the entire dataset before the first action. The data efficiency benefits of \ArtReplay hinge on balancing the relation: it must be fine enough to limit data usage, but coarse enough to maintain meaningful informativeness.

\section{Theoretical Analysis of \ArtReplay}
\label{sec:benefits}

In this section, we highlight the theoretical benefits of \ArtReplay.  We start by proving that for the broad class of base algorithms that satisfy \emph{independence of irrelevant data} (\iidata), a novel property we propose, \ArtReplay incurs identical regret to \Full---thereby improving data efficiency without compromising performance.  This emphasizes that \ArtReplay Pareto-dominates \Full across the two metrics of interest.  We then highlight a regret improvement, emphasizing that \ArtReplay improves the regret over \Ignorant based on the amount of {\em used} historical data.  
Finally, we offer a case study of \ArtReplay applied to the \MonUCB algorithm from \cref{sec:ar_case_study_mon_ucb} to $(i)$ establish that \ArtReplay has arbitrarily better data efficiency than \Full, and $(ii)$ quantify the value of the historical data.  However, we believe these results extend to a variety of bandit models under arbitrary base algorithms.

\subsection{\iidata and Regret Couplings}
\label{sec:iidata_coupling}

As defined, it is not immediately clear how to analyze the regret of \ArtReplay, or even compare its performance to \Full.  To enable regret analysis, we introduce a new property for bandit algorithms, \emph{independence of irrelevant data} (\iidata), which states that when an algorithm is about to take an action, providing additional data about \emph{other} unrelated actions will not change the algorithm's decision.

\begin{definition}[Independence of irrelevant data]
\label{def:iidata}
A deterministic algorithm $\basealgo$ together with an information relation $\rel$ satisfies the {\bf independence of irrelevant data} (\iidata) property if whenever action~$\action$ is proposed by the algorithm~$\basealgo$ based on the current data $(\action = \basealgo(\H))$ then
\begin{align}
    \basealgo(\H) = \basealgo(\H \cup \H^\prime)
\end{align}
for any dataset $\H^\prime$ satisfying $\H^\prime \cap \{\actiongeneric \in \actionset : \actiongeneric \rel_{\H} \basealgo(\H) \} = \emptyset$, that is, containing data from actions which are not related to $\basealgo(\H)$. 
\end{definition}

While it is not a priori clear whether \iidata is satisfied by existing algorithms in the literature, in \cref{sec:algorithms} we highlight regret-optimal algorithms under a variety of domains which satisfy this property (including \MonUCB from \cref{sec:ar_case_study_mon_ucb}).  However, it is easy to see that any base algorithm satisfies \iidata under the trivial relation where $\actiongeneric \rel \actiongeneric'$ for any two actions $\actiongeneric, \actiongeneric'$.

More generally, \iidata is a natural robustness property for an algorithm to satisfy and is conceptually analogous to the independence of irrelevant alternatives (IIA) axiom in computational social choice, often cited as a desiderata used to evaluate voting rules \citep{arrow1951social}.  In the existing bandit literature, there has been a narrow focus on only finding regret-optimal algorithms. We propose that \iidata is another desirable property that implies ease and robustness for optimally and efficiently incorporating historical data.  We moreover conjecture that \iidata is {\em necessary} for ensuring that the regret gains from incorporating historical data is monotone increasing, since without the property the adversary can, in an algorithm-dependent way, augment $\history$ to include samples and influence the algorithm to select a sub-optimal action.  We show that \iidata avoids these issues by ensuring \ArtReplay and \Full have {\em identical regret}.  


\begin{theorem}[Regret Coupling of \ArtReplay to \Full]
\label{thm:coupling}
Suppose that base algorithm~$\basealgo$ with information relation $\rel$ satisfies \iidata. Then for any problem instance, horizon~$T$, and historical dataset~$\history$ we have the following: 
\begin{align*}
    \pi^{\ArtReplay(\basealgo)}_t & \eqdist \pi^{\Full(\basealgo)}_t\\
    \regret \left( T, \pi^{\ArtReplay(\basealgo)}, \history \right) & \eqdist \regret \left( T, \pi^{\Full(\basealgo)}, \history \right) \ .
\end{align*}
\end{theorem}

\noindent {\em Proof Sketch.}
The proof of this result uses the reward stack model for analyzing the performance of bandit algorithms developed in \citet{lattimore2020bandit} and follows the spirit of \citet{wilson1996generating}. We provide a sample path coupling over the reward observations for $\pi^{\ArtReplay(\basealgo)}$ and $\pi^{\Full(\basealgo)}$ to show that $\pi^{\ArtReplay(\basealgo)}_t = \pi^{\Full(\basealgo)}$ for all $t$, and hence have the same distribution. 

We establish this result by induction over $t$, where we explicitly use the \iidata property to prove that although \ArtReplay uses a subset of the historical data, its selected action is identical to that of \Full. 
 Indeed, consider the case when $t = 1$, then \ArtReplay builds up a set $\Hon_1$ from the historical data such that $(\history \setminus \Hon_1) \cap \{\actiongeneric \rel_{\Hon_1} \basealgo(\Hon_1)\} = \emptyset$.  However, by the fact that $\basealgo$ satisfies \iidata it must be that $\pi^{\Full(\basealgo)}_1 = \basealgo(\history) = \basealgo(\Hon_1) = \pi^{\ArtReplay(\basealgo)}_1$. 
 The remainder of the proof proceeds similarly and is delegated to \cref{app:proof_coupling}.

Next we show a {\em regret improvement} for using \ArtReplay.  By design, \ArtReplay takes as input a base algorithm and information relation and \emph{simulates} the historical data in a replay buffer.  As such, it seems natural that we can bound the regret of \ArtReplay (as a random variable), to a counterfactual simulation of the algorithm applied to a {\em longer time horizon} (where the length corresponds to the amount of used historical data).  Indeed, we can show the following:
\begin{theorem}[Regret Improvement of \ArtReplay]
\label{thm:generic_regret_gain}
Let $\basealgo$ be any base algorithm, and denote by $\Delta_{\min} = \min_{\actiongeneric \in \history, a \neq \aopt} \OPT - \mu(\actiongeneric)$ the smallest suboptimality gap in the historical dataset.  Let $\Hon_T$ be a random variable corresponding to the subset of $\history$ used by $\pi^{\ArtReplay(\basealgo)}$ after the $T$ timesteps.  Further let $\Hon_T(\dagger)$ correspond to the subset of $\Hon_T$ containing data for actions which are not optimal.  Then we have that:
\begin{equation}
\label{eq:generic_regret_gain}
    \Exp{\regret \left( T, \pi^{\ArtReplay(\basealgo)}, \history \right)} \leq \Exp{\regret \left( T + \left| \Hon_T \right|, \pi^{\Ignorant(\basealgo)} \right) - \Delta_{\min} \left| \Hon_T(\dagger) \right|}.
\end{equation}
\end{theorem}
The proof of \cref{thm:generic_regret_gain} again follows using a regret coupling and is detailed in \cref{app:proofs}. 

Note here that the base algorithm $\Pi$ needs to be identical on the $T$-length and $T + |\Hon_T|$ problems, but this is avoided in UCB-style algorithms through using $T+H$ in the logarithmic terms.
The expectation is taken over the randomness in the reward observations for both {\em offline} and {\em online} data. Note that \cref{eq:generic_regret_gain} depends on both $|\Hon_T|$ and $|\Hon_T(\dagger)|$, which are random variables and are a non-trivial function of the choice of $\basealgo$ alongside the information relation.  In \cref{thm:reg_gains} we instantiate this theorem for \MonUCB to obtain a more interpretable regret improvement bound.
However, \cref{thm:generic_regret_gain} already highlights the benefits of incorporating historical data into a regret-optimal bandit algorithm.  In a typical regret-optimal algorithm, $\regret(T, \pi^{\Ignorant(\basealgo)})$ is a sublinear function of~$T$ for {\em any} choice of $T > 0$.  Thus the first term in \cref{eq:generic_regret_gain} scales sublinearly with respect to $|\Hon_T|$, while the second term scales {\em linearly} with respect to $|\Hon_T|$.  This highlights the benefit of incorporating historical data for a general regret-optimal base algorithm.


\subsection{Case Study: \MonUCB}
\label{sec:mon_ucb_benefits}

We now offer a case study of \ArtReplay applied to \MonUCB from \cref{sec:ar_case_study_mon_ucb}. We show that \MonUCB satisfies \iidata and hence \ArtReplay has identical regret guarantees to \Full, and also that \ArtReplay has arbitrarily better data efficiency than \Full. We then quantify the value of the historical data in terms of regret improvement.  Many of the results will extend to arbitrary base algorithms in more general settings.

For the rest of this section, we consider the relation where $\actiongeneric \rel \actiongeneric'$ if and only if $\actiongeneric = \actiongeneric'$.  We start by recalling \cref{thm:coupling} which establishes that \Full and \ArtReplay have identical regret for \iidata algorithms.  Indeed, by showing \MonUCB satisfies \iidata, we obtain the following:

\begin{theorem}
\label{thm:mon_ucb_iid}
\MonUCB satisfies \iidata.  Hence
\[
\regret \left( T, \pi^{\ArtReplay(\MonUCB)}, \history \right) \eqdist \regret \left( T, \pi^{\Full(\MonUCB)}, \history \right) \ .
\]
\end{theorem}
This is the {\em finest} relation such that \MonUCB satisfies \iidata.  Not all $K$-armed bandit algorithms satisfy \iidata under this relation. For example, Thompson Sampling (TS)~\citep{russo2018tutorial} samples arms according to the posterior probability that they are optimal. TS does not satisfy \iidata: data from actions other than the one chosen will adjust the posterior distribution, and hence will adjust the selection probabilities as well. However, as we show in \cref{fig:k-armed-ts-ids}, \ArtReplay still achieves empirical gains over \Full with Thompson sampling, despite not satisfying \iidata.


We next highlight that \ArtReplay is robust to {\em spurious data}, where the historical data has excess samples coming from poor-performing actions.  Spurious data imposes computational challenges since the \Full approach will pre-process the full historical dataset regardless of the observed rewards or the inherent value of the historical data.  In contrast, \ArtReplay will only use the amount of data useful for learning.  \textbf{This allows \ArtReplay to have {\em arbitrarily better data efficiency} than \Full!}

Note that \ArtReplay imposes minimal computational and storage overhead on top of existing algorithms, simply requiring a data structure to verify whether $(\history \setminus \Hon_t) \cap \{\actiongeneric \in \actionset : \actiongeneric \rel_{\Hon_t \cup \H_t} \Pi(\Hon_t \cup \H_t)\} = \emptyset$.  However, this can be done efficiently with hashing techniques.  We make this formal in the following theorem:
\begin{theorem}
\label{thm:comp_gains}
For every $H \in \mathbb{N}$ there exists a historical dataset $\history$ with $|\history| = H$ where the data efficiency of $\pi^{\Full(\MonUCB)}$ is linear in $\Omega(H+T)$ whereas the data efficiency of $\pi^{\ArtReplay(\MonUCB)}$ is only $O(T+\min\{\sqrt{T},\log(T) / \min_a \Delta(a)^2\})$.
\end{theorem}
\begin{rproof}
    \cref{thm:comp_gains} is a corollary of \cref{thm:comp_gains_psi} (appendix) under the selection of $\psi(\lambda) = \lambda^2 / 8$.
\end{rproof}


For the $K$-armed bandit, \Full requires $O(K)$ storage to maintain estimates for each arm, while a naive implementation of \ArtReplay requires $O(K+H)$ storage since the entire historical dataset needs to be stored.  However, effective hashing can address the extra $H$ factor.  
We also note that most practical bandit applications (including content recommendation systems and the poaching prevention setting discussed) incorporate historical data obtained from database systems. This historical data will be stored regardless of the algorithm being employed, and so the key consideration is computational requirements.  

Lastly, to complement the data efficiency improvements of \ArtReplay applied to \MonUCB, we also characterize the regret improvement due to incorporating historical data relative to a data ignorant algorithm. As the regret coupling implies that \ArtReplay achieves the same regret as \Full, we can simply analyze the regret improvement gained by \Full to characterize the regret of \ArtReplay.

\begin{theorem}
\label{thm:reg_gains}
Let $H_\actiongeneric$ be the number of data points in $\history$ for each action $\actiongeneric \in [K]$.  Then the regret of Monotone UCB with historical dataset~$\history$ is:
\begin{align*}
        & \Exp{\regret \left( T, \pi^{\ArtReplay(\MonUCB)}, \history \right) \mid (H_a)_{a \in \actionset} } \leq O \Big( \sum\limits_{\actiongeneric \in [K] : \Delta_\actiongeneric \neq 0} \max\big\{0, \tfrac{\log(T)}{\Delta(a)} - H_a \Delta(a) \big\} \Big) \ .
\end{align*}
\end{theorem}
\begin{rproof}
    \cref{thm:reg_gains} is a corollary of \cref{thm:reg_gains_psi} under the choice of $\psi(\lambda) = \lambda^2 / 8$.
\end{rproof}

\cref{thm:comp_gains} together with \cref{thm:reg_gains} helps demonstrate the advantage of using \ArtReplay over \Full: we improve data efficiency (and hence, computational complexity) while maintaining an equally improved regret guarantee. 
Moreover, it highlights the impact historical data can have on the regret.  If $|H_a| \geq \log(T) / \Delta(a)^2$ for each action~$a$, then the regret of the algorithm will be logarithmic with $T$ (similar insights have been seen in the context of online pricing~\citep{bu2020online}).  This result also reduces to the standard \MonUCB guarantee (\cref{thm:mon_ucb_regret}) when $\history = \emptyset$.  
We remark that there are no existing regret lower bounds for incorporating historical data in bandit algorithms. Our regret guarantees seem likely to be optimal, as $\log(T)/\Delta(a)$ has been shown to be the minimax regret for gathering sufficient information to eliminate a suboptimal arm $a$, and $H_a \Delta(a)$ naturally represents a reduction in regret of $\Delta(a)$ for each pull of arm $a$ in the dataset. 
 \cref{thm:reg_gains} also highlights interesting insights into a measure of {\em usefulness} of the historical data. Historical data for the optimal action has no impact on regret (since playing that action online does not incur any regret).  Moreover, it is better to have {\em more} data on {\em more suboptimal} actions.

It remains an interesting direction for future work to understand the optimality of \Full as well as how to optimally incorporate historical data in settings when \Full may not be optimal, for example, due to contaminated or stale historical data~\citep{cheung2024leveraging}.  Our key contribution is to highlight a simple, intuitive, and implementable approach through \ArtReplay which matches the performance of \Full while simultaneously achieving better data efficiency.

\section{\iidata Algorithms}
\label{sec:algorithms}

\cref{thm:coupling,thm:generic_regret_gain} show that for \iidata algorithms, \ArtReplay achieves identical regret as \Full and offers a regret improvement with respect to the amount of used historical data (the data efficiency).  The regret performance of \ArtReplay therefore hinges on the choice of the base algorithm and information relation; as a result, we suggest using a regret-optimal base algorithm and relation that satisfies \iidata.

In addition to \ArtReplay being optimal for the \MonUCB algorithm as we proved in \cref{sec:mon_ucb_benefits}, here we show that the widely used {\em regret-optimal} bandit algorithms which use the {\em upper confidence bound} approach (i.e., choosing greedily from optimistic estimates) can be adapted to satisfy \iidata while retaining their regret guarantees.  We provide example regret-optimal algorithms for metric bandits and \CMABCRA (\cref{sec:def_metric_bandit,sec:def_cmab_cra}).  Other algorithms, such as for linear response bandits~\citep{bastani2020online,goldenshluger2013linear}, can be similarly modified to satisfy \iidata.  Additional details are in \cref{sec:app_algorithms}.

\subsection{\iidata in Metric Bandits}
\label{sec:metric_bandits_iidata}

We first consider the standard metric bandit set-up of \cref{sec:def_metric_bandit}. We assume that $\actionset$ is a compact metric space with metric~$d$, and that the reward function~$\mu$ is $L$-Lipschitz continuous with respect to the metric, i.e.:
\[
\abs{\mu(\actiongeneric) - \mu(\actiongeneric')} \leq Ld(\actiongeneric, \actiongeneric').
\]
\noindent Incorporating historical data efficiently is difficult in continuous action spaces. The key issues are that the computation and storage costs grow with the size of the historical dataset, and the estimation and discretization are done independently of the quality of the reward.  Two natural approaches to address continuous actions are to $(i)$~discretize the action space based on the data using nearest-neighbor estimates, or $(ii)$~learn a regression of the mean reward using available data. When using \Full, both approaches are data inefficient in light of {\em spurious data} --- when excessive data is collected from poor-performing actions. Fixed discretization-based algorithms will unnecessarily process and store a large number of discretizations in low-performing regions of the space (\cref{fig:visualize_discretization}); regression-based methods require additional compute resources to learn an accurate predictor of the mean reward in irrelevant regions.

We consider fixed and adaptive discretization algorithms, establishing that a straightforward modification of existing UCB algorithms are regret-optimal and satisfy \iidata under an appropriate information relation~\citep{kleinberg2019bandits}.  Here we use the natural information relation, where at a high level $\actiongeneric \rel_\H \actiongeneric'$ so long as they both fall within the same discretized region.

The algorithm maintains a collection of regions~$\P_t$ of $\actionset$ which covers~$\actionset$. For fixed discretization, $\P_t$ is fixed at the start of learning as an $\gamma-$covering of the action set~$\actionset$; with adaptive discretization, it is refined over the course of learning based on observed data. We use $\reg \in \P_t$ to denote a {\em region} of the action space.  The monotone discretization algorithms track the following: \textit{(i)} $\mubar_t(\reg)$ for the estimated mean reward of region~$\reg$, \textit{(ii)} $n_t(\reg)$ for the number of times $\reg$ has been selected, and $\textit{(iii)}$~$\UCB_t(\reg)$ for an upper confidence bound estimate of the reward.  
At each timestep~$t$, our algorithm performs three steps.  First for the {\em action selection} we select the region which maximizes $\UCB_t(\reg)$ and pick any $\actiongeneric \in \reg$ to play arbitrarily.  Afterwards, we {\em update parameters} by incrementing $n_t(\reg)$ by one, update $\mubar_t(\reg)$ based on observed data, and set \[\UCB_{t+1}(\reg) = \min \left\{ \UCB_t(\reg), \mubar_t(\reg) + b(n_t(\reg)) \right\}\] for some appropriate bonus term $b(\cdot)$.  The UCB update enforces monotonicity in the $\UCB$ values similar to $\MonUCB$ and is required to preserve \iidata.  Lastly, in the case of adaptive-discretization, we potentially {\em re-partition} the space.  We split a region when the confidence in its estimate $b(n_t(\reg))$ is smaller than the diameter of the region and replace it with new regions of half the diameter. This condition may seem independent of the {\em quality} of a region, but since it is incorporated into a learning algorithm, the number of samples in a region is correlated with its reward.

Under appropriately defined bonus term $b(\cdot)$ and selection of $\gamma$ (see \citet{kleinberg2019bandits,sinclair2021adaptive} and \cref{sec:app_algorithms_metric_bandits} for further details), it can be shown that these algorithms are regret-optimal.  However, we also have:
\begin{theorem}
\label{thm:iid_metric_bandits}
The fixed discretization algorithm under the relation where $\actiongeneric \rel \actiongeneric'$ whenever $d(\actiongeneric, \actiongeneric') \leq \gamma$ is \iidata.  The adaptive discretization algorithm under the relation where $\actiongeneric \rel_\H \actiongeneric'$ whenever $\actiongeneric$ and $\actiongeneric'$ belong to the same region $\reg$ over data $\H$ satisfies \iidata.
\end{theorem}

\cref{thm:iid_metric_bandits} allows us to match the regret performance of \Full while simultaneously avoiding reading in the entire historical dataset to improve data efficiency.

\subsection{\iidata in \CMABCRA}
\label{sec:cmab_iid_example}

Now we extend the algorithms from \cref{sec:metric_bandits_iidata} to the \CMABCRA set-up (see \cref{sec:def_cmab_cra}) which are used in the computational results in \cref{sec:experiments}.  First, we outline two assumptions on the underlying problem.  The first is the standard nonparametric assumption, highlighting that the resource space~$\pointset$ is a compact metric space with metric $d_\pointset$, diameter $\dmax$, and that the true underlying reward function~$\reward$ is Lipschitz with respect to~$d_\pointset$.  Indeed we have:
\[
\abs{\reward(\point, \effortgeneric) - \reward(\point', \effortgeneric)} \leq L d_\pointset(\point, \point').
\]
\noindent The next assumption assumes access to an oracle for solving optimization problems of the form of \cref{eq:cmab_opt} for arbitrary choice of reward functions $\reward(\point, \effortgeneric)$.  We can relax this assumption to instead assume that there exists a randomized approximation oracle by appropriately shifting the regret benchmark.  See \citet{zuo2021combinatorial} which considers this in a discrete setting.

Our algorithms directly modify those in \cref{sec:metric_bandits_iidata} to maintain a separate discretization for each allocation $\effortgeneric \in \effortset$.  Indeed, the algorithms are UCB style, where the selection rule optimizes over the combinatorial action space (\cref{eq:action_set}) through a discretization of~$\pointset$.  For each allocation $\effortgeneric \in \effortset$, the algorithm maintains a collection of regions $\P_t^\effortgeneric$ of $\pointset$ which covers $\pointset$. For fixed discretization, $\P_t^\effortgeneric$ is fixed at the start of learning; with adaptive discretization it is refined throughout learning based on observed data. 

The algorithm tracks the following: \textit{(i)} $\mubar_t(\reg, \effortgeneric)$ for the estimated mean reward of region~$\reg$ at allocation~$\effortgeneric$, \textit{(ii)} $n_t(\reg, \effortgeneric)$ for the number of times $\reg$ has been selected at allocation~$\effortgeneric$, and $\textit{(iii)}$~$\UCB_t(\reg, \effortgeneric)$ for an upper confidence bound estimate of the reward.  
At each timestep~$t$, our algorithm performs three steps:
\begin{enumerate}
\item {\bf Action selection}: Greedily select at most $N$~regions in $\P_t^\effortgeneric$ to maximize $\UCB_t(\reg, \effortgeneric)$ subject to the budget constraint (see \cref{eq:selection} in the appendix). Note that we must ensure that each region is selected at only a \emph{single} allocation value.  This is solved using a standard knapsack formulation.
\item {\bf Update parameters}: For each of the selected regions, increment $n_t(\reg, \effortgeneric)$ by one, update $\mubar_t(\reg, \effortgeneric)$ based on observed data, and set \[\UCB_{t+1}(\reg, \effortgeneric) = \min \left\{ \UCB_t(\reg, \effortgeneric), \mubar_t(\reg, \effortgeneric) + b(n_t(\reg, \effortgeneric)) \right\}\] for some appropriate bonus term $b(\cdot)$.  The UCB update enforces monotonicity in the $\UCB$ estimates similar to $\MonUCB$ and is required to preserve \iidata.
\item {\bf Re-partition}: We split a region when the confidence in its estimate $b(n_t(\reg, \effortgeneric))$ is smaller than the diameter of the region. In \cref{fig:discretization_visual} (appendix) we highlight how the adaptive discretization algorithm hones in on high-reward regions without knowing the reward function before learning.
\end{enumerate}
Due to space, we describe the algorithms at a high level here and defer details to \cref{sec:app_algorithms_cmab_cra}. These algorithms modify existing approaches applied to \CMABCRA in the bandit and reinforcement learning literature, which have been shown to be regret-optimal~\citep{xu2021dual,sinclair2021adaptive}. Additionally, these approaches are \iidata under the natural discretization-based information relation (where we additionally account for the semi-bandit feedback in the style of \cref{sec:information_relation}).
\begin{theorem}
\label{thm:iid_example}
The fixed and adaptive discretization algorithms when using a ``greedy'' solution to solve the action selection rule have the \iidata property.
\end{theorem}
\noindent {\em Proof Sketch.} In this proof, we require the algorithm to use the standard ``greedy approximation'', which is a knapsack problem in the \CMABCRA set-up~\citep{williamson2011design}.  In general, this introduces additional approximation ratio limitations.  However, under additional assumptions on the mean reward function $\mu(\point, \effortgeneric)$, the greedy solution is provably optimal.  For example, optimality of the greedy approximation holds when $\mu(\point, \effortgeneric)$ is piecewise linear and monotone, or more broadly when $\mu(\actiongeneric)$ is submodular.  See \cref{sec:app_algorithms_cmab_cra} for more discussion.  

\noindent Note that these assumptions hold for our application of green security to prevent wildlife poaching that we present in \cref{sec:experiments}.

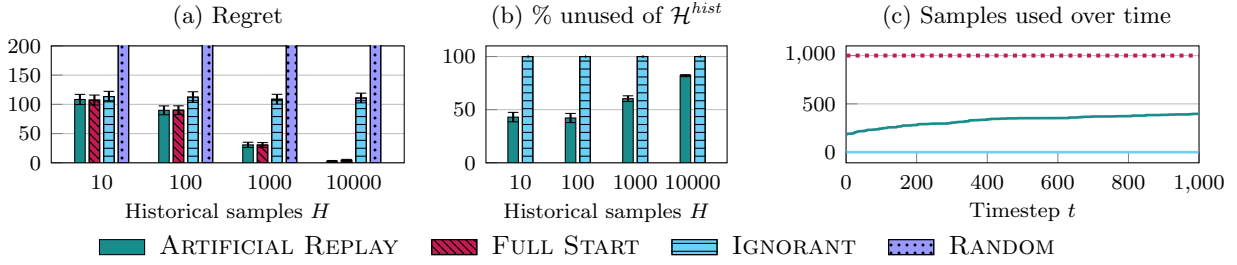
\begin{figure*}[!t]
    \centering
    \begin{tikzpicture}
\pgfplotsset{
  width=0.38\linewidth,
  height=0.23\linewidth,
  xtick pos=left,
  ytick pos=left,
  tick label style={font=\scriptsize},
  ymajorgrids=true,
  xlabel shift=-2pt,
  title style={yshift=-.7ex},
}

\begin{axis}[
  at={(.34\linewidth, 0\linewidth)}, 
  width=0.31\linewidth,
  table/col sep=comma,
  ybar=1.5pt,
  bar width=.14cm,
  symbolic x coords={10, 100, 1000, 10000},
  xticklabels={10, 100, {1,000}, {10,000}},
  xtick style={draw=none},
  x tick label style={yshift=.7ex}, 
  xtick={data},
  xlabel={\footnotesize{Historical samples~$H$}},
  title={\small{(b)~\% unused of $\history$}},
  ylabel style={align=center},
  ymin=0, 
  enlarge x limits={0.2},
  error bars/y dir=both, 
  error bars/y explicit,  
  error bars/error bar style={color=black, thick},
]

\addplot+[teal!90!white, draw=black, area legend,
] table [x=H, y=artificial_replay, y error=artificial_replay_sem] {data/k_armed_b1_unused_h_percentage.csv}; 
\addplot+[cyan!50!white, draw=black, area legend,
postaction={pattern=horizontal lines,}
] table [x=H, y=ignorant, y error=ignorant_sem] {data/k_armed_b1_unused_h_percentage.csv}; 
\end{axis}

\begin{axis}[
  at={(0\linewidth, 0\linewidth)}, 
  table/col sep=comma,
  ybar=1.5pt,
  bar width=.14cm,
  symbolic x coords={10, 100, 1000, 10000},
  xticklabels={10, 100, {1,000}, {10,000}},
  xtick style={draw=none},
  x tick label style={yshift=.7ex}, 
  xtick={data},
  xlabel={\footnotesize{Historical samples~$H$}},
  title={\small{(a)~Regret}},
  ylabel style={align=center},
  legend style={
    at={(0.8\linewidth, -.07\linewidth)},
    legend columns=4,
    font=\small,draw=none,fill=none,
    /tikz/every even column/.append style={column sep=10pt}},
  ymin=0, ymax=200,
  enlarge x limits={0.2},
  error bars/y dir=both, 
  error bars/y explicit,  
  error bars/error bar style={color=black, thick},
]

\addplot+[teal!90!white, draw=black, area legend,
] table [x=H, y=artificial_replay, y error=artificial_replay_sem] {data/k_armed_b1_regret.csv}; \addlegendentry{~\ArtReplay~~~~};
\addplot+[purple!90!white, draw=black, area legend,
postaction={pattern=north west lines,}
] table [x=H, y=historical, y error=historical_sem] {data/k_armed_b1_regret.csv}; \addlegendentry{~\Full~~~~};
\addplot+[cyan!50!white, draw=black, area legend,
postaction={pattern=horizontal lines,}
] table [x=H, y=ignorant, y error=ignorant_sem] {data/k_armed_b1_regret.csv}; \addlegendentry{~\Ignorant~~~~};
\addplot+[blue!40!white, draw=black, area legend,
postaction={pattern=dots,}
] table [x=H, y=random, y error=random_sem] {data/k_armed_b1_regret.csv}; \addlegendentry{~\Random~~~~};
\end{axis}

\begin{axis}[ 
  at={(.64\textwidth, 0\textwidth)},
  title={\small{~(c)~Samples used over time}},
  xlabel={\footnotesize{Timestep~$t$}},
  xlabel shift=-4pt,
  x tick label style={yshift=-.2ex}, 
  ymax=1100,
  xmin=0, xmax=1000,
]

\addplot [teal!90!white, line width=1pt, mark=none] table [x=t, y=artificial_replay, col sep=comma] {data/k_armed_b1_avg_used_h_k10_H1000.csv}; 

\addplot [purple!90!white, line width=1.5pt, dotted, mark=none] table [x=t, y=historical, col sep=comma] {data/k_armed_b1_avg_used_h_k10_H1000.csv}; 

\addplot [cyan!50!white, line width=1pt, mark=none] table [x=t, y=ignorant, col sep=comma] {data/k_armed_b1_avg_used_h_k10_H1000.csv}; 

\end{axis}

\end{tikzpicture}
    \caption{(\textit{$K$-armed)} Increasing the number of historical samples~$H$ leads \Full to use unnecessary data, particularly as $H$ gets very large. \ArtReplay achieves equal performance in terms of regret (plot~a) while using less than half the historical data (plot~b). In plot~(c) we see that with $H=1{,}000$ historical samples, \ArtReplay uses (on average) 117 historical samples before taking its first online action. The number of historical samples used increases at a decreasing rate, using only 396 of $1{,}000$ total samples by the horizon~$T$.
    Results are shown on the $K$-armed bandit setting with $K=10$ and horizon $T=1{,}000$.
    }
    \label{fig:k-armed-b1-spurious-data}
\end{figure*}

\begin{figure*}[!t]
    \centering
     \begin{tikzpicture}
\pgfplotsset{
  width=0.38\linewidth,
  height=0.23\linewidth,
  xtick pos=left,
  ytick pos=left,
  xtick style={draw=none},
  title style={yshift=-.7ex},
  ylabel style={align=center},
  tick label style={font=\scriptsize},
  x tick label style={yshift=.7ex}, 
  ymajorgrids=true,
  xlabel shift=-2pt,
  enlarge x limits={0.2},
}

\begin{axis}[
  at={(0.34\linewidth, 0\linewidth)}, 
  width=0.31\linewidth,
  table/col sep=comma,
  ybar,
  bar width=.14cm,
  symbolic x coords={0.1, 0.5, 0.7, 0.9},
  xtick={data},
  title={\small{(b)~\% unused of $\history$}},
  xlabel={\footnotesize{Fraction bad arms}},
  legend style={
    at={(1.7,1.4)},
    legend columns=4,
    font=\footnotesize,draw=none,fill=none},
  ymin=0, 
  error bars/y dir=both, 
  error bars/y explicit,  
  error bars/error bar style={color=black, thick},
]

\addplot+[teal!90!white, draw=black, area legend,
] table [x=frac_vals, y=adaptive_artificial_replay, y error=adaptive_artificial_replay_sem] {data/adaptive_unused_h_frac.csv}; 
\addplot+[cyan!50!white, draw=black, area legend,
postaction={pattern=horizontal lines,}
] table [x=frac_vals, y=adaptive_ignorant, y error=adaptive_ignorant_sem] {data/adaptive_unused_h_frac.csv}; 
\end{axis}

\begin{axis}[
  at={(.64\linewidth, 0\linewidth)}, 
  table/col sep=comma,
  ybar,
  bar width=.14cm,
  symbolic x coords={0.1, 0.5, 0.7, 0.9},
  xtick={data},
  title={\small{(c)~Number of regions~$\reg$}},
  xlabel={\footnotesize{Fraction bad arms}},
  legend style={
    at={(0.5\linewidth,.25\linewidth)},
    legend columns=3,
    font=\small,draw=none,fill=none},
  ymin=0, 
  error bars/y dir=both, 
  error bars/y explicit,  
  error bars/error bar style={color=black, thick},
]

\addplot+[teal!90!white, draw=black, area legend,
] table [x=frac_vals, y=adaptive_artificial_replay, y error=adaptive_artificial_replay_sem] {data/adaptive_n_regions.csv}; 
\addplot+[purple!90!white, draw=black, area legend,
postaction={pattern=north west lines,}
] table [x=frac_vals, y=adaptive_historical, y error=adaptive_historical_sem] {data/adaptive_n_regions.csv}; 
\addplot+[cyan!50!white, draw=black, area legend,
postaction={pattern=horizontal lines,}
] table [x=frac_vals, y=adaptive_ignorant, y error=adaptive_ignorant_sem] {data/adaptive_n_regions.csv}; 
\end{axis}

\begin{axis}[
  at={(0\linewidth, 0\linewidth)}, 
  table/col sep=comma,
  ybar=1.5pt,
  bar width=.14cm,
  symbolic x coords={0.1, 0.5, 0.7, 0.9},
  xtick={data},
  title={\small{(a)~Regret}},
  xlabel={\footnotesize{Fraction bad arms}},
  legend style={
    at={(0.8\linewidth, -.07\linewidth)},
    legend columns=4,
    font=\small,draw=none,fill=none,
    /tikz/every even column/.append style={column sep=10pt}},
  ymin=0, ymax=600,
  error bars/y dir=both, 
  error bars/y explicit,  
  error bars/error bar style={color=black, thick},
]

\addplot+[teal!90!white, draw=black, area legend,
] table [x=frac_vals, y=adaptive_artificial_replay, y error=adaptive_artificial_replay_sem] {data/adaptive_regret.csv}; \addlegendentry{~\ArtReplay~~~~};
\addplot+[purple!90!white, draw=black, area legend,
postaction={pattern=north west lines,}
] table [x=frac_vals, y=adaptive_historical, y error=adaptive_historical_sem] {data/adaptive_regret.csv}; \addlegendentry{~\Full~~~~};
\addplot+[cyan!50!white, draw=black, area legend,
postaction={pattern=horizontal lines,}
] table [x=frac_vals, y=adaptive_ignorant, y error=adaptive_ignorant_sem] {data/adaptive_regret.csv}; \addlegendentry{~\Ignorant~~~~};
\addplot+[yellow!90!white, draw=black, area legend,
postaction={pattern=crosshatch dots,}
] table [x=frac_vals, y=regression, y error=regression_sem] {data/adaptive_regret.csv}; \addlegendentry{~\Regressor~~~~};
\end{axis}

\end{tikzpicture}
    \caption{\textit{(\CMABCRA)} Holding $H=10{,}000$ constant, we increase the fraction of historical data samples on bad arms (bottom 20\% of rewards). The plots show (a)~regret, (b)~$\%$ of unused historical data, and (c)~number of discretized regions in partition~$\P$. \ArtReplay enables significantly improved runtime and reduced storage while matching the performance of \Full. Results on the \CMABCRA setting with adaptive discretization on the quadratic domain.}
    \label{fig:imbalanced-data-adaptive}
\end{figure*}
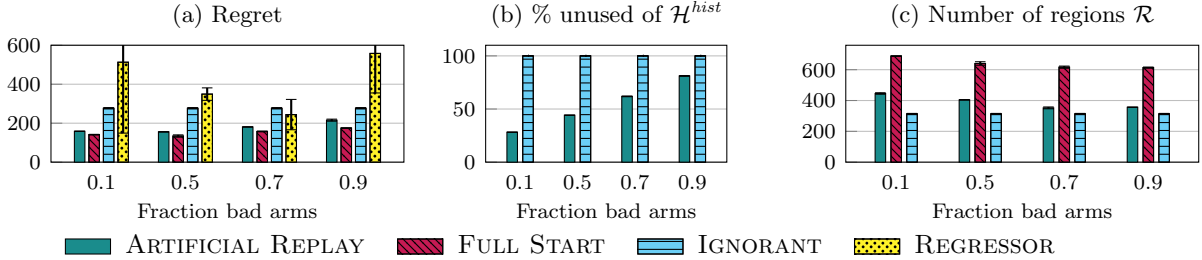
\section{Empirical Benefits of \ArtReplay}
\label{sec:experiments}

\begin{figure*}[!t]
    \centering
    
    \input{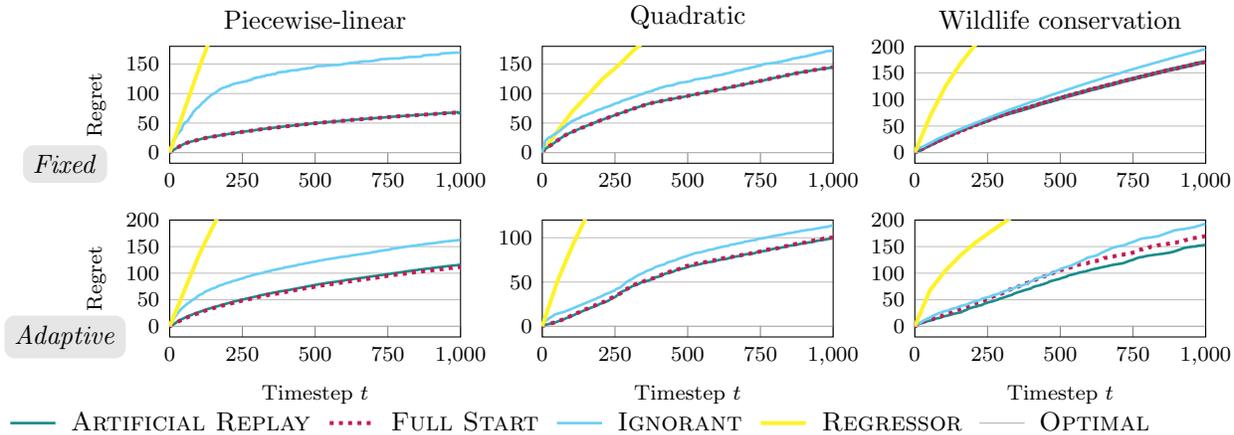}
    \caption{\textit{(\CMABCRA)} Cumulative regret ($y$-axis; lower is better) across time~$t \in [T]$. \ArtReplay performs equally as well as \Full across all domain settings, including both fixed discretization (top row) and adaptive discretization (bottom). \Regressor performs quite poorly.}
    \label{fig:performance}
\end{figure*}

We complement the theoretical advantages of \ArtReplay by showing that, in practice, our meta-algorithm achieves identical performance to \Full while significantly improving data efficiency. 
Experiment details and additional results are available in \cref{sec:experiments_appendix}.
\footnote{Code to reproduce the experiments is available at \mbox{\url{https://github.com/lily-x/artificial-replay}}.}
We conduct experiments on two of the bandit models described in \cref{sec:bandit-models}: finite $K$-armed bandits and \CMABCRA, using both fixed and adaptive discretization. For the continuous combinatorial setting, we provide two stylized domains, a piecewise-linear and a quadratic reward function. 

\paragraph{Green Security Domain.}  To emphasize the practical benefit of \ArtReplay, we evaluate on a real-world resource allocation setting for biodiversity conservation. Here, we introduce a new model for green security games with continuous actions by using adaptive discretization. The \CMABCRA model can be used to specify the green security setting from \citet{xu2021dual}, where the space~$\pointset$ represents a protected area and $\effortset$ is the discrete set of patrol resources to allocate, such as number of ranger hours per week with a budget~$B$ of 40~hours. This formulation generalizes to a continuous-space model of the landscape, instead of the artificial fixed discretization that was considered in prior work consisting of $1 \times 1$ sq.~km regions of the park. The reward function $\reward(\point, \effortgeneric)$ then models the probability of observing a snare at location $\point$ at effort level $\effortgeneric$.

We study real ranger patrol data from Murchison Falls National Park, shared as part of a collaboration with the Uganda Wildlife Authority and the Wildlife Conservation Society. We use historical patrol observations to build the history~$\history$; we analyze these historical observations in detail in \cref{sec:experiments_appendix} to show that this dataset exhibits both spurious data and imbalanced coverage.

\paragraph{Baselines.} We compare \ArtReplay against \Ignorant and \Full approaches.  In the $K$-armed model, we use \MonUCB as the base algorithm.  In \CMABCRA we use fixed and adaptive discretization as well as \Regressor, a neural network learner that is a regression-based approach analogue to \Full. \Regressor is initially trained on the entire historical dataset, then iteratively retrained after $128$ new samples are collected. We compute for each setting the performance of an \Optimal action based on the true rewards and a \Random baseline that acts randomly while satisfying the budget constraint.

\paragraph{Results.} 
The results in \cref{fig:performance} empirically validate our theoretical result from \cref{thm:coupling}: the performance of \ArtReplay is identical to that of \Full, and reduces regret considerably compared to the naive \Ignorant approach. 
We evaluate the regret (compared to \Optimal) of each approach across time $t \in [T]$. Concretely, we consider the three domains of piecewise-linear reward, quadratic reward, and green security with continuous space~$\pointset = [0,1]^2$, $N=5$ possible action components, a budget $B=2$, and $3$ levels of effort. We include $H=300$ historical data points. See \cref{fig:k-armed-b3-spurious-data} (appendix) for regret and analysis of historical data use on the $K$-armed bandit. Results are averaged over 60 iterations with random seeds, with standard error plotted. 


Not only does \ArtReplay achieve equal performance, but its data efficiency improvements over \Full are clear even on practical problem sizes. As we increase historical data from $H = \{10; 100; 1{,}000; 10{,}000 \}$ in \cref{fig:k-armed-b1-spurious-data}, the proportion of irrelevant data increases. Our method achieves equal performance, overcoming the previously unresolved challenge of \emph{spurious data}, while \Full suffers from arbitrarily worse data efficiency (\cref{thm:comp_gains}). With $10{,}000$ historical samples and a time horizon of $1{,}000$, we see that 58.2\% of historical samples are irrelevant to producing the most effective policy.



When faced with \emph{imbalanced data coverage}, the benefits of \ArtReplay become clear---most notably in the continuous action setting with adaptive discretization. In \cref{fig:imbalanced-data-adaptive}, as we increase the number of historical samples on bad regions (bottom 20th percentile of reward), the additional data require finer discretization, leading to arbitrarily worse storage and computational complexity for \Full with equal regret. 
In \cref{fig:imbalanced-data-adaptive}(c), we see that with 10\% of data on bad arms, \ArtReplay requires only 446 regions~$\reg$ compared to 688 used by \Full; as we get more spurious data and that fraction increases to 90\%, then \ArtReplay requires only 356 regions while \Full still stores 614 regions.

\section{Conclusion}

We present \ArtReplay, a meta-algorithm that modifies {\em any base bandit algorithm} to efficiently harness historical data. As we show, under a widely applicable \iidata condition, the regret of \ArtReplay is distributionally identical to that of a full warm-start approach, while also guaranteeing significantly better data efficiency. 

Our experimental results highlight the advantage of using \ArtReplay over \Full on a variety of base algorithms, applied to $K$-armed and continuous combinatorial bandit models. These advantages even hold for base algorithms such as Thompson sampling and Information Directed Sampling (IDS) that do not exhibit \iidata.  

Directions for future work include \textit{(i)}~find \iidata algorithms in other bandit domains such as linear contextual bandits, \textit{(ii)}~incorporate the \ArtReplay approach into reinforcement learning, \textit{(iii)}~motivate \iidata by highlighting it is robust to adversarial data generation, and \textit{(iii)}~provide theoretical bounds showing that \ArtReplay has \emph{optimal} data efficiency when incorporating historical data.

\bibliographystyle{informs2014} 
\bibliography{ref} 


\newpage
\crefalias{section}{appendix}

\begin{APPENDICES}


\renewcommand{\arraystretch}{0.75} 

\begin{table}[!ht]
\caption{List of common notations}
    \centering
    \begin{tabular}{ll}
\toprule
\textbf{Symbol} & \textbf{Definition} \\ \midrule
\multicolumn{2}{c}{\textit{Problem setting specifications}} \\
\midrule
$\actionset$ & Feasible action space \\
$\rewarddist$ & Reward distribution, i.e., $\rewarddist : \actionset \rightarrow \Delta([0,1])$ \\
$T$ & Time horizon \\
$\D$ & An arbitrary set of $(\actiongeneric, r)$ pairs \\
$\history, H$ & Historical data available to algorithm and number of historical datapoints \\
$\regret(\pi, T, \history)$ & Cumulative regret for an algorithm~$\pi$ on $T$~timesteps with historical data $\history$ \\
$\actiongeneric$ & Generic action $\actiongeneric \in \actionset$ \\
$\actionhistt{j}$ & Historical action at index~$j$ in history~$\history$ \\
$\actiont{t}$ & Selected action chosen at timestep~$t$ \\
$\obsreward_t$ & Reward observed at timestep~$t$ from action~$\actiont{t}$ \\
$\basealgo$ & Base algorithm that maps ordered $(\actiongeneric, \obsreward)$ pairs to a distribution over $\actionset$ \\
$\rel_{\H}$ & Information relation over $\actionset$ based on the data $\H$ \\
$\pi^{\ArtReplay}$ & \ArtReplay framework and its resulting policy \\
$\pi^{\Full}$ & \Full: Full warm starting operation and its resulting policy\\
$\pi^{\Ignorant}$ & \Ignorant: Original policy ignoring historical data\\
\midrule
\multicolumn{2}{c}{\textit{\CMABCRA specification}}\\
\midrule
\CMABCRA & Combinatorial Multi-Armed Bandit for Continuous Resource Allocation \\
$\pointset, \effortset$ & The continuous resource space, and (discrete) space of allocation values \\
$d_\pointset, \dmax$ & Metric over $\pointset$ and the diameter of $\pointset$ \\
$N, \epsilon$ & Maximum number of regions which can be selected, and the minimum distance \\
$\rewarddist(\point, \effort)$ & Reward distribution for a particular allocation $(\point, \effort)$ \\
$\reward(\point, \effort)$ & Mean reward for a particular allocation $(\point, \effort)$ \\
\midrule
\multicolumn{2}{c}{\textit{Monotone UCB}}\\
\midrule
$\mubar_t(\actiongeneric), n_t(\actiongeneric)$ & Mean reward estimates and number of samples for action $\actiongeneric \in [K]$ \\
$\UCB_t(\actiongeneric)$ & Upper confidence bound estimate of action~$\actiongeneric$ \\
\midrule
\multicolumn{2}{c}{\textit{Fixed Discretization}}\\
\midrule
$\P$ & Fixed $\epsilon$ covering of $\pointset$ \\
$\mubar_t(\reg, \effortgeneric), n_t(\reg, \effortgeneric)$ & Mean reward estimates and number of samples for region $\reg \in \P$ \\
$\UCB_t(\reg, \effortgeneric)$ & Upper confidence bound estimate of $\reward(\reg, \effortgeneric)$ \\
$b(t)$ & Bonus term (confidence radius) for a region which has been selected $t$ times \\
\midrule
\multicolumn{2}{c}{\textit{Adaptive Discretization}}\\
\midrule
$\P_t^\effortgeneric$ & Partition of space~$\pointset$ at timestep~$t$ for allocation $\effortgeneric \in \effortset$ \\
$\mubar_t(\reg, \effortgeneric), n_t(\reg, \effortgeneric)$ & Mean reward estimates and number of samples for region $\reg \in \P_t^\effortgeneric$ \\
$\UCB_t(\reg, \effortgeneric)$ & Upper confidence bound estimate of $\reward(\reg, \effortgeneric)$ \\
$r(\reg)$ & Diameter of a region~$\reg$ \\
$b(t)$ & Bonus term for a region which has been selected $t$ times \\
\bottomrule
    \end{tabular}
\end{table}

\section{IIData Algorithms}
\label{sec:app_algorithms}

\subsection{\MonUCB Algorithms for $K$-Armed Bandits}
\label{sec:ucb_algorithms_appendix}

In this section we detail the Monotone UCB (denoted \MonUCB) style algorithms for the $K$-Armed Bandit problem of \cref{sec:example_finite}, which are regret-optimal and satisfy the \iidata property under the relation where $\actiongeneric \rel \actiongeneric'$ if and only if $\actiongeneric = \actiongeneric'$.  These algorithms are derived from the $\psi$-UCB based algorithms from~\citet{bubeck2012regret}. For full pseudocode of the algorithm see \cref{alg:psi_ucb_algorithm}.  We describe the algorithm without incorporating historical data (where its counterpart involving historical data are derived by treating this as the base algorithm~$\basealgo$ and appealing to \ArtReplay or \Full described in \cref{sec:online_wrapper}).

\begin{algorithm}[tb]
   \caption{Monotone $\psi$-UCB (\MonUCB)}\label{alg:psi_ucb_algorithm}
\begin{algorithmic}[1]
    \State \textbf{Input}: Convex function $\psi$
   \State Initialize $n_1(\actiongeneric) = 0$, $\mubar_1(\actiongeneric) = 1$, and $\UCB_1(\actiongeneric) = 1$ for each $\actiongeneric \in [K]$
   \For{$t = \{1, 2, \ldots\}$}
   \State Let $\actiont{t} = \argmax_{\actiongeneric \in [K]} \UCB_t(\actiongeneric)$
   \State Receive reward~$\obsrewardt{t}$ sampled from $\rewarddist(\actiont{t})$
   \For{all $\actiongeneric \neq \actiont{t}$}
   \State $n_{t+1}(\actiongeneric) = n_t(\actiongeneric)$
   \State $\mubar_{t+1}(\actiongeneric) = \mubar_t(\actiongeneric)$
   \State $\UCB_{t+1}(\actiongeneric) = \UCB_t(\actiongeneric)$
   \EndFor
   \State $n_{t+1}(\actiont{t}) = n_t(\actiont{t}) + 1$
   \State $\mubar_{t+1}(\actiont{t}) = (n_t(\actiont{t}) \mubar_t(\actiont{t}) + \obsrewardt{t}) / n_{t+1}(\actiont{t})$
   \State $\UCB_{t+1}(\actiont{t}) = $ \mbox{$\min \left\{\UCB_t(\actiont{t}), \mubar_{t+1}(\actiont{t}) + (\psi^*)^{-1}\left(\frac{2\log(TK)}{n_{t+1}(A_t)}\right) \right\}$}
   \EndFor
\end{algorithmic}
\end{algorithm}

Concretely, \MonUCB tracks the following, akin to standard UCB approaches: \textit{(i)}~$\mubar_t(\actiongeneric)$ for the estimated mean reward of action~$\actiongeneric \in [K]$, \textit{(ii)}~$n_t(\actiongeneric)$ for the number of times the action~$\actiongeneric$ has been selected by the algorithm prior to timestep~$t$, and \textit{(iii)}~$\UCB_t(\actiongeneric)$ for an upper confidence bound estimate for the reward of action~$\actiongeneric$. 
At every timestep~$t$, the algorithm picks the action~$\actiont{t}$ which maximizes $\UCB_t(\actiongeneric)$ (breaking ties deterministically).  After observing~$\obsreward_t$, we update our counts $n_{t+1}(a)$ and estimates $\mubar_{t+1}(a)$. 

Similar to \citet{bubeck2012regret}, we assume the algorithm has access to a convex function $\psi$ satisfying the following assumption:
\begin{assumption}
\label{ass:convex_psi}
For all $\lambda \geq 0$ and $a \in \A$ we have that:
\begin{align*}
    \log\left( \Exp{e^{\lambda(\rewarddist(a) - \mu(a))}}\right) & \leq \psi(\lambda) \\
     \log\left( \Exp{e^{\lambda(\mu(a) - \rewarddist(a))}}\right) & \leq \psi(\lambda).
\end{align*}
\end{assumption}
\noindent Note that this is trivially satisfied when $\psi(\lambda) = \lambda^2 / 8$ using Hoeffding's lemma and the assumption that the rewards are bounded in $[0,1]$.

We further denote by $\psi^*$ as the Legendre-Fenchel transform of $\psi$, defined via:
\[
    \psi^*(\epsilon) = \sup_{\lambda} \lambda \epsilon - \psi(\epsilon).
\]
With this, we are now ready to define how the confidence bounds are computed as follows:
\[
\UCB_{t+1}(\actiont{t}) = \UCB_t(\actiont{t}), \quad \mubar_{t+1}(\actiont{t}) + (\psi^*)^{-1}\left(\frac{2\log(T)}{n_{t+1}(A_t)}\right)
\]
Note that again when $\psi(\lambda) = \lambda^2 / 8$ then we again recover the standard UCB1 algorithm (as presented in \cref{sec:algorithms}).  Next we state the more general version of the theorems from \cref{sec:ar_case_study_mon_ucb} and \cref{sec:benefits} applied to \MonUCB with arbitrary convex function $\psi$.  We omit all proofs from the results here, and defer them to \cref{app:proofs}.

First we highlight that \MonUCB for an arbitrary convex function $\psi$ satisfying \cref{ass:convex_psi} satisfies the \iidata property under the relation where $\actiongeneric \rel \actiongeneric'$ if and only if $\actiongeneric = \actiongeneric'$.  We similarly establish that the base algorithm is regret-optimal, recovering the results established in \citet{bubeck2012regret}.

\begin{theorem}
\label{thm:mon_ucb_iid_psi}
The \MonUCB base algorithm with arbitrary convex function $\psi$ satisfying \cref{ass:convex_psi} and under the relation where $\actiongeneric \sim \actiongeneric'$ if and only if $\actiongeneric = \actiongeneric'$ 
satisfies the \iidata property.  Moreover, for $\Delta(\actiongeneric) = \max_{\actiongeneric^\prime} \mu(\actiongeneric^\prime) - \mu(\actiongeneric)$ then
\begin{align*}
    \Exp{\regret(T, \pi^{\Ignorant(\MonUCB)}, \history)} = O(\sum\nolimits_{\actiongeneric} \Delta(a) \log(T) / \psi^*(\Delta(\actiongeneric))). 
\end{align*}
\end{theorem}
\noindent Note that \cref{thm:mon_ucb_iid,thm:mon_ucb_regret} are corollaries of \cref{thm:mon_ucb_iid_psi} under the selection of $\psi(\lambda) = \lambda^2 / 8$.  The next result highlights that there exists a historical dataset such that \Full has unbounded computational complexity (as the number of datapoints goes to infinity), whereas \ArtReplay has bounded computational complexity.  Combined with the regret coupling, this highlights that \ArtReplay Pareto-dominates \Full in terms of regret and computational complexity.

\begin{theorem}
\label{thm:comp_gains_psi}
Consider the relation where $\actiongeneric \sim \actiongeneric'$ if and only if $\actiongeneric = \actiongeneric'$.  
For every $H \in \mathbb{N}$ there exists a historical dataset $\history$ with $|\history| = H$ where the runtime of $\pi^{\Full(\MonUCB)} = \Omega(H+T)$ whereas the runtime of $\pi^{\ArtReplay(\MonUCB)} = O(T+\min\{\sqrt{T},\log(T) / \min_a \Delta(a)^2\})$.
\end{theorem}
\noindent We again note that \cref{thm:comp_gains} is a corollary of \cref{thm:comp_gains_psi} under the selection of $\psi(\lambda) = \lambda^2 / 8$.  The final result establishes a regret improvement for \ArtReplay scaling as the size of the historical data used.

\begin{theorem}
\label{thm:reg_gains_psi}
Consider the relation where $\actiongeneric \sim \actiongeneric'$ if and only if $\actiongeneric = \actiongeneric'$. 
Let $H_\actiongeneric$ be the number of datapoints in $\history$ for each action $\actiongeneric \in [K]$.  Then the regret of Monotone UCB with historical dataset~$\history$ is:
\begin{align*}
        & \Exp{\regret(T, \pi^{\ArtReplay(\MonUCB)}, \history)} \\
         & \leq O \Big( \sum\limits_{\actiongeneric \in [K] : \Delta_\actiongeneric \neq 0} \max\big\{0, \frac{\Delta(a)\log(T)}{\psi^*(\Delta(a))} - H_a \Delta(a) \big\} \Big).
\end{align*}
\end{theorem}
\noindent We again note that \cref{thm:reg_gains} is a corollary of \cref{thm:reg_gains_psi} under the choice of $\psi(\lambda) = \lambda^2 / 8$.

\subsection{Discretization Algorithms for Metric Bandits}
\label{sec:app_algorithms_metric_bandits}

In this section we detail a fixed and adaptive discretization algorithm for metric bandits from \cref{sec:metric_bandits_iidata} which satisfies the \iidata property.  These serve as a straightforward modification to the algorithms from \citet{kleinberg2019bandits}, where similar to \MonUCB we adjust the confidence terms to be monotone decreasing.  We start by summarizing the main algorithm sketch, before highlighting the key details and differences between the two.  Full pseudocode is omitted, since in \cref{sec:app_algorithms_cmab_cra} we will extend the approach to the more complicated \CMABCRA set-up.  We here describe the algorithm without incorporating historical data, since its counterpart involving historical data can be used by treating this as the base algorithm~$\basealgo$ and appealing to \ArtReplay or \Full described in \cref{sec:online_wrapper}.

Our algorithms are Upper Confidence Bound (UCB) style as the selection rule picks an action $\actiongeneric \in \actionset$ over a discretization of $\actionset$.  Both algorithms are parameterized by the time horizon~$T$ and a value $\delta \in (0,1)$.  The algorithm maintains a collection of regions $\P_t$ of $\actionset$.  Each element $\reg \in \P_t$ is a region (i.e. subset of $\actionset$) with diameter~$r(\reg)$.  For the fixed discretization variant, $\P_t$ is fixed at the start of $t=1$.  In the adaptive discretization algorithm, this partitioning is refined over the course of learning in a {\em data-driven manner}.

For each time period $t$ the algorithm maintains three tables linear with respect to the number of regions in the partition.  This includes an \emph{upper confidence value} $\UCB_t(\reg)$ for the true $\mu(\actiongeneric)$ value for points $\actiongeneric$ in~$\reg$ (which is initialized to be one), determined based on an estimated mean $\mubar_t(\reg)$ and $n_t(\reg)$ for the number of times $\reg$ has been selected by the algorithm in timesteps up to~$t$.  The latter is incremented every time $\reg$ is played and is used to construct the bonus term.   At a high level, our algorithms perform two steps in each iteration~$t$: select an action via the {\em selection rule} and then {\em update parameters}.  In addition, the adaptive discretization algorithm will {\em re-partition} the space.  In order to define the steps, we first introduce some definitions and notation.

The \emph{confidence radius} (or bonus) of region~$\reg$ is defined:
\begin{align}
    b(n_t(\reg)) = 2 \sqrt{\frac{2 \log(T / \pfail) + 2L}{n_t(\reg)}}
\end{align}
corresponding to the uncertainty in estimates due to the stochastic nature of the rewards.  Then the steps of the algorithm are defined as follows:
\begin{enumerate}
    \item \textbf{Selection rule}: Pick the action $\reg_t = \argmax_{\reg \in \P_t} \UCB_t(\reg)$.  Play any action $\actiont{t} \in \reg_t$.
    \item \textbf{Update parameters}: Increment $n_t(\reg_t)$ by~$1$, update $\mubar_t(\reg_t)$ based on observed data, and update $\UCB_t(\reg_t)$ while ensuring monotonicity, i.e.:
    \[
    \UCB_{t+1}(\reg_t) = \min(\UCB_{t}(\reg_t), \mu_{t+1}(\reg_t) + b(n_{t+1}(\reg_t)).
    \]
\end{enumerate}
In the fixed discretization algorithm we set $\P_1$ to be a $\gamma$-covering of $\actionset$.  For the adaptive discretization, the algorithm additionally decides whether to update the partition.  This is done via:
\begin{enumerate}
    \item[3] \textbf{Re-partition the space}: Let $\reg_t$ denote the selected region and $r(\reg_t)$ denote its radius. We split when $n_t(\reg_t) \geq \left(1/r(\reg_t)\right)^2$.  We then cover $\reg$ with new regions $\reg^1, \ldots, \reg^n$ which form an $\frac{1}{2}r(\reg_t)$-Net of $\reg_t$.  We call $\reg_t$ the \textit{parent} of these new balls and each child ball inherits all values from its parent.  We then add the new balls $\reg^1, \ldots, \reg^n$ to $\P_t$ to form the partition for the next timestep $\P_{t+1}$.
\end{enumerate}

\paragraph{Information relations.}  Under the fixed discretization algorithm we take the natural information relation where $\actiongeneric \rel \actiongeneric'$ whenever $d(\actiongeneric, \actiongeneric') \leq \gamma$, where $\gamma$ is bandwidth parameter for the discretization $\P_t$.  For the adaptive discretization algorithm, we use a {\em history}-dependent information ordering.  Here, $\actiongeneric \rel_\H \actiongeneric'$ whenever $\actiongeneric$ and $\actiongeneric'$ fall into the same region in the adaptive partition of $\actionset$ determined based on $\H$.  Here we emphasize that the adaptive discretization generated is only a function of the sequence of selected actions (with no dependence on $t$ or the order for which they are read).

\subsection{Discretization Algorithms for \CMABCRA}
\label{sec:app_algorithms_cmab_cra}

\begin{figure}[!t]
        \centering
            \scalebox{1}{
            \tikzset{every picture/.style={line width=0.75pt}} 
  
\tikzset {_holmkmsfn/.code = {\pgfsetadditionalshadetransform{ \pgftransformshift{\pgfpoint{0 bp } { 0 bp }  }  \pgftransformrotate{-45 }  \pgftransformscale{2 }  }}}
\pgfdeclarehorizontalshading{_07soqxhih}{150bp}{rgb(0bp)=(0,0.35,0.71);
rgb(40.089285714285715bp)=(0,0.35,0.71);
rgb(62.5bp)=(0.86,0.2,0.13);
rgb(100bp)=(0.86,0.2,0.13)}

  
\tikzset {_tte39it5o/.code = {\pgfsetadditionalshadetransform{ \pgftransformshift{\pgfpoint{0 bp } { 0 bp }  }  \pgftransformrotate{-45 }  \pgftransformscale{2 }  }}}
\pgfdeclarehorizontalshading{_o64v3fts8}{150bp}{rgb(0bp)=(0,0.35,0.71);
rgb(40.089285714285715bp)=(0,0.35,0.71);
rgb(62.5bp)=(0.86,0.2,0.13);
rgb(100bp)=(0.86,0.2,0.13)}

  
\tikzset {_wtpe7f1s9/.code = {\pgfsetadditionalshadetransform{ \pgftransformshift{\pgfpoint{0 bp } { 0 bp }  }  \pgftransformrotate{-270 }  \pgftransformscale{2 }  }}}
\pgfdeclarehorizontalshading{_owa0zuidu}{150bp}{rgb(0bp)=(0,0.35,0.71);
rgb(37.5bp)=(0,0.35,0.71);
rgb(62.5bp)=(0.86,0.2,0.13);
rgb(100bp)=(0.86,0.2,0.13)}
\tikzset{every picture/.style={line width=0.75pt}} 

\begin{tikzpicture}[x=0.75pt,y=0.75pt,yscale=-1,xscale=1]

\path  [shading=_07soqxhih,_holmkmsfn] (232,5) -- (372,5) -- (372,145) -- (232,145) -- cycle ; 
 \draw   (232,5) -- (372,5) -- (372,145) -- (232,145) -- cycle ; 

\path  [shading=_o64v3fts8,_tte39it5o] (61,5) -- (201,5) -- (201,145) -- (61,145) -- cycle ; 
 \draw   (61,5) -- (201,5) -- (201,145) -- (61,145) -- cycle ; 

\draw  [draw opacity=0] (61,5) -- (201,5) -- (201,145) -- (61,145) -- cycle ; \draw   (81,5) -- (81,145)(101,5) -- (101,145)(121,5) -- (121,145)(141,5) -- (141,145)(161,5) -- (161,145)(181,5) -- (181,145) ; \draw   (61,25) -- (201,25)(61,45) -- (201,45)(61,65) -- (201,65)(61,85) -- (201,85)(61,105) -- (201,105)(61,125) -- (201,125) ; \draw   (61,5) -- (201,5) -- (201,145) -- (61,145) -- cycle ;
\path  [shading=_owa0zuidu,_wtpe7f1s9] (4,5) -- (24,5) -- (24,145) -- (4,145) -- cycle ; 
 \draw   (4,5) -- (24,5) -- (24,145) -- (4,145) -- cycle ; 

\draw   (232,5) -- (372,5) -- (372,145) -- (232,145) -- cycle ;
\draw   (232,5) -- (302,5) -- (302,75) -- (232,75) -- cycle ;
\draw   (232,75) -- (302,75) -- (302,145) -- (232,145) -- cycle ;
\draw   (302,75) -- (372,75) -- (372,145) -- (302,145) -- cycle ;
\draw   (302,75) -- (337,75) -- (337,110) -- (302,110) -- cycle ;
\draw   (302,110) -- (337,110) -- (337,145) -- (302,145) -- cycle ;
\draw   (337,75) -- (372,75) -- (372,110) -- (337,110) -- cycle ;
\draw   (337,110) -- (355,110) -- (355,126) -- (337,126) -- cycle ;
\draw   (355,110) -- (372,110) -- (372,126) -- (355,126) -- cycle ;
\draw   (355,126) -- (372,126) -- (372,145) -- (355,145) -- cycle ;
\draw   (337,126) -- (355,126) -- (355,145) -- (337,145) -- cycle ;
\draw   (337,110) -- (346,110) -- (346,118) -- (337,118) -- cycle ;
\draw   (346,110) -- (355,110) -- (355,118) -- (346,118) -- cycle ;
\draw   (337,118) -- (346,118) -- (346,126) -- (337,126) -- cycle ;

\draw (27,135) node [anchor=north west][inner sep=0.75pt]   [align=left] {0};
\draw (27,1) node [anchor=north west][inner sep=0.75pt]   [align=left] {1};

\end{tikzpicture}}
        \caption{Comparison of a fixed (middle) and adaptive (right) discretization on a two-dimensional resource set~$\pointset$ for a fixed allocation level~$\effortgeneric$.  The underlying color gradient corresponds to the mean reward $\mu(\point, \effortgeneric)$ with red corresponding to higher value and blue to lower value (see figure on left for legend).  The fixed discretization algorithm is forced to {\em explore uniformly} across the entire resource space.  In contrast, the adaptive discretization algorithm is able to maintain a {\em data efficient} representation, even without knowing the underlying mean reward function a priori.}        
\label{fig:discretization_visual}
\end{figure}
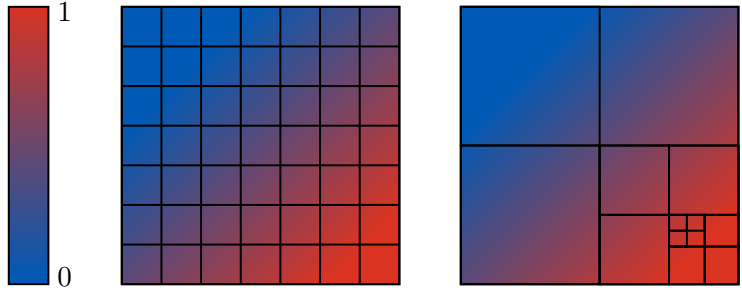

In this section we detail the fixed and adaptive discretization algorithms for \CMABCRA which satisfy the \iidata property, and serve as a combinatorial extension to the algorithms developed in \cref{sec:app_algorithms_metric_bandits}.  We start off by summarizing the main algorithm sketch, before highlighting the key details and differences between the two.  For pseudocode of the adaptive discretization algorithm see \cref{alg:adaptive_discretization}.  We describe the algorithm without incorporating historical data (where its counterpart involving historical data can be used by treating this as the base algorithm~$\basealgo$ and appealing to \ArtReplay or \Full described in \cref{sec:online_wrapper}).

Our algorithms are Upper Confidence Bound (UCB) style as the selection rule maximizes \cref{eq:cmab_opt} approximately over a discretization of $\pointset$.  Both algorithms are parameterized by the time horizon~$T$ and a value $\delta \in (0,1)$.
For each allocation $\effortgeneric \in \effortset$ the algorithm maintains a collection of regions $\P_t^\effortgeneric$ of~$\pointset$.  Each element $\reg \in \P_t^\effortgeneric$ is a region with diameter~$r(\reg)$.  For the fixed discretization variant, $\P_t^\effortgeneric$~is fixed at the start of learning.  In the adaptive discretization algorithm, this partitioning is refined over the course of learning in an {\em data-driven manner}.

For each time period~$t$, the algorithm maintains three tables linear with respect to the number of regions in the partitions~$\P_t^\effortgeneric$ and size of $\effortset$.  For every region $\reg \in \P_t^\effortgeneric$ we maintain an \emph{upper confidence value} $\UCB_t(\reg, \effortgeneric)$ for the true $\mu(\reg, \effortgeneric)$ value for points in~$\reg$ (which is initialized to be one), determined based on an estimated mean $\mubar_t(\reg, \effortgeneric)$ and $n_t(\reg, \effortgeneric)$ for the number of times $\reg$ has been selected by the algorithm in timesteps up to~$t$.  The latter is incremented every time $\reg$ is played, and is used to construct the bonus term.   At a high level, our algorithms perform two steps in each iteration~$t$: select an action via the {\em selection rule} and then {\em update parameters}.  In addition, the adaptive discretization algorithm will {\em re-partition} the space.  In order to define the steps, we first introduce some definitions and notation.

Let $t_\reg = n_t(\reg, \effortgeneric)$ be the number of times the algorithm has selected region~$\reg \in \P$ at allocation~$\effortgeneric$ by time~$t$.  The \emph{confidence radius} (or bonus) of region~$\reg$ is defined:
\begin{align}
    b(t_\reg) = 2 \sqrt{\frac{2 \log(T / \pfail)}{t_\reg}}
\end{align}
corresponding to the uncertainty in estimates due to stochastic nature of the rewards.  Lastly, the UCB value for a region~$\reg$ is computed as $\UCB_t(\reg, \effortgeneric) = \min\{\UCB_{t-1}(\reg, \effortgeneric), \mubar_t(\reg, \effortgeneric) + b(n_t(\reg, \effortgeneric)\}$.   This enforces monotonicity in the UCB
estimates, similar to \MonUCB, and is required for the \iidata property to hold.

At each timestep~$t$, the algorithm \emph{selects regions} according to the following optimization procedure:
\begin{align}
    \max_{z(\reg, \effortgeneric) \in \{0,1\}} & \sum_{\effortgeneric \in \effortset} \sum_{\reg \in \P_t^\effortgeneric} \UCB_t(\reg, \effortgeneric) \cdot z(\reg, \effortgeneric) \label{eq:selection} \\
    \text{s.t. } & \sum_{\effortgeneric \in \effortset} \sum_{\reg \in \P_t^\effortgeneric} \effortgeneric \cdot  z(\reg, \effortgeneric) \leq B \nonumber \\
    & \sum_{\effortgeneric \in \effortset} \sum_{\reg \in \P_t^\effortgeneric} z(\reg, \effortgeneric) \leq N \nonumber \\ 
    & z(\reg, \effortgeneric) + \sum_{\effortgeneric' \neq \effortgeneric} \sum_{\tilde{\reg} \in \P_t^{\effortgeneric'}, \reg \subset \tilde{\reg}} z(\tilde{\reg},\effortgeneric') \leq 1 \quad \forall \effortgeneric, \reg \in \P_t^\effortgeneric \nonumber
\end{align}
The objective encodes the goal of maximizing the upper confidence bound terms.  The first constraint encodes the budget limitation, and the second that at most $N$ regions can be selected.  The final constraint is a technical one, essentially requiring that for each region $\reg \in \pointset$, the same region is not selected at different allocation amounts.  In the supplementary code base we provide an efficient implementation which avoids this step by ``merging'' the trees appropriately so that each region contains a vector of estimates of $\mubar_t(\reg, \effortgeneric)$ for each $\effortgeneric \in \effortset$.  Based on the optimal solution, the final action~$\action$ is taken by picking $(\point, \effortgeneric)$ for each~$\reg$ such that $\point \in \reg$ and $z(\reg, \effortgeneric) = 1$.  We lastly note that this optimization problem is a well-known ``knapsack'' problem with efficient polynomial-time approximation guarantees.  It also has a simple ``greedy'' solution scheme, which iteratively selects the regions with largest $\UCB_t(\reg, \beta) / \beta$ ratio (i.e. the so-called ``bang-per-buck'').  See \citet{williamson2011design} for more discussion.

After subsequently observing the rewards for the selected regions, we increment $t_\reg = n_t(\reg, \effortgeneric)$ by one for each selected region, update $\mubar_t(\reg, \effortgeneric)$ accordingly with the additional datapoint, and compute $\UCB_t(\reg, \effortgeneric)$.
Then the two rules are defined as follows:
\begin{enumerate}
    \item \textbf{Selection rule}: Greedily select at most $N$ regions subject to the budgetary constraints which maximizes $\UCB_t(\reg, \effortgeneric)$ following \cref{eq:selection}. 
    \item \textbf{Update parameters}: For each of the selected regions~$\reg$, increment $n_t(\reg, \effortgeneric)$ by~$1$, update $\mubar_t(\reg, \effortgeneric)$ based on observed data, and update $\UCB_t(\reg, \effortgeneric)$ while ensuring monotonicity.
\end{enumerate}

\subsubsection{Fixed Discretization}
\label{sec:fixed_disc}

This algorithm is additionally parameterized by a discretization level~$\gamma$.  The algorithm starts by maintaining a $\gamma$-covering of~$\pointset$, which we denote as $\P_t^\effortgeneric = \P$ for all $t \in [T]$ and $\effortgeneric \in \effortset$.  For the information relation, we define:
\[
((\point_1, \effortgeneric_1), \ldots, (\point_N, \effortgeneric_N)) \rel ((\point'_1, \effortgeneric'_1), \ldots, (\point_N', \effortgeneric_N'))
\]
if there exists an index $j$ with $\effortgeneric_j = \effortgeneric_j'$ and $d_\pointset(\point_j, \point'_j) \leq \gamma$.  Note that this combines the continuous and semi-bandit relations described in \cref{sec:information_relation}.

\paragraph{Information relation.} 
For the information relation, we define:
\[
((\point_1, \effortgeneric_1), \ldots, (\point_N, \effortgeneric_N)) \rel ((\point'_1, \effortgeneric'_1), \ldots, (\point_N', \effortgeneric_N'))
\]
if there exists an index $j$ with $\effortgeneric_j = \effortgeneric_j'$ and $d_\pointset(\point_j,\point'_j) \leq \gamma$.  This combines the continuous and semi-bandit relation models outlined in \cref{sec:information_relation}.

\subsubsection{Adaptive Discretization}
\label{sec:ada_discretization}

\begin{algorithm}[tb]
   \caption{Adaptive Discretization for \CMABCRA (\adaalgo)} \label{alg:adaptive_discretization}
\begin{algorithmic}[1]
\State \textbf{Input:} Resource set $\pointset$, effort set $\effortset$, timesteps $T$, and probability of failure $\pfail$
   \State Initiate $|\effortset|$ partitions $\P_1^\effortgeneric$ for each $\effortgeneric \in \effortset$, each containing a single region with radius $\dmax$ and $\mubar_1^\effortgeneric$ estimate equal to $1$
   
   \For{each timestep $\{t \gets 1, \ldots, T\}$}
   \State Select the regions by the selection rule \cref{eq:selection}
   \State For each selected region~$\reg$ (regions where $z(\reg, \effortgeneric) = 1$), add $(\actiongeneric, \effortgeneric)$ to $\actiont{t}$ for any $a \in \reg$
   \State Play action $\actiont{t}$ in the environment
   \State Update parameters: $t = n_{t+1}(\reg_{\text{sel}}, \effortgeneric) \gets n_t(\reg_{\text{sel}}, \effortgeneric) + 1$ for each selected region $\reg_{\text{sel}}$ with $z(\reg_{\text{sel}}, \effortgeneric) = 1$, and update $\mubar_t(\reg_{\text{sel}}, \effortgeneric)$ accordingly with observed data

   \If{$n_{t+1}(\reg, \effortgeneric) \geq \left( \frac{\dmax}{r(\reg)}\right)^2$ and $r(\reg) \geq 2 \epsilon$}
   
   \State \textsc{Split Region}$(\reg, \effortgeneric, t)$
   \EndIf
   \EndFor

\end{algorithmic}
\end{algorithm}
\begin{algorithm}[!t]
   \caption{\textsc{Split Region} (Sub-Routine for \adaalgo)}
\begin{algorithmic}[1]
\State \textbf{Input:} Region $\reg$, allocation amount $\effortgeneric$, timestep $t$
    \State Set $\reg_1, \ldots, \reg_n$ to be an $\frac{1}{2}r(\reg)$-packing of $\reg$, and add each region to the partition $\P_{t+1}^{\effortgeneric}$
    \State Initialize parameters $\mubar_t(\reg_i, \effortgeneric)$ and $n_t(\reg_i, \effortgeneric)$ for each new region $\reg_i$ to inherent values from the parent region $\reg$
\end{algorithmic}
\end{algorithm}

For each effort level $\effortgeneric \in \effortset$ the algorithm maintains a collection of regions $\P_t^\effortgeneric$ of $\pointset$ which is refined over the course of learning for each timestep~$t$. Initially, when $t=1$, there is only one region in each partition $\P_1^\effortgeneric$ which has radius~$1$ containing~$\pointset$.

The key differences from the fixed discretization are two-fold.  First, the \emph{confidence radius} or bonus of region~$\reg$ is defined via:
\[
    b(t) = 2 \sqrt{\frac{2 \log(T / \pfail)}{t}} + \frac{2L\dmax}{\sqrt{t}} \ .
\]
The first term corresponds to uncertainty in estimates due to stochastic nature of the rewards, and the second is the discretization error by expanding estimates to all points in the region.  

Second, after selecting an action and updating the estimates for the selected regions, the algorithm additionally decides whether to update the partition.  This is done via:
\begin{enumerate}
    \item[3] \textbf{Re-partition the space}: Let $\reg$ denote any selected ball and $r(\reg)$ denote its radius. We split when $r(\reg) \geq 2 \epsilon$ and $n_t(\reg, \effortgeneric) \geq \left(\dmax/r(\reg)\right)^2$.  We then cover $\reg$ with new regions $\reg_1, \ldots, \reg_n$ which form an $\frac{1}{2}r(\reg)$-Net of $\reg$.  We call $\reg$ the \textit{parent} of these new balls and each child ball inherits all values from its parent.  We then add the new balls $\reg_1, \ldots, \reg_n$ to $\P_t^\effortgeneric$ to form the partition for the next timestep $\P_{t+1}^\effortgeneric$.
\end{enumerate}

\paragraph{Information relation.} For the information relation, we define:
\[
((\point_1, \effortgeneric_1), \ldots, (\point_N, \effortgeneric_N)) \rel_\H ((\point'_1, \effortgeneric'_1), \ldots, (\point_N', \effortgeneric_N'))
\]
if there exists an index $j$ with $\effortgeneric_j = \effortgeneric_j'$ and $\point_j$ and $\point'_j$ both belong to the same region in the adaptive discretization dictated by $\H$.  Note that this combines the continuous and semi-bandit relations described in \cref{sec:information_relation}.

\paragraph{Benefits of Adaptive Discretization.}  The fixed discretization algorithm cannot adapt to the underlying structure in the problem since the discretization is fixed prior to learning.  This causes increased computational, storage, and sample complexity since each individual region must be explored in order to obtain optimal regret guarantees.  Instead, our adaptive discretization algorithm adapts the discretization in a data-driven manner to reduce unnecessary exploration.  The algorithms keep a fine discretization across important parts of the space, and a coarser discretization across unimportant regions.  See \cref{fig:discretization_visual} for a sample adaptive discretization observed, and \cite{kleinberg2010regret,sinclair2021adaptive} for more discussion on the benefits of adaptive discretization over fixed.

\paragraph{Implementation of Adaptive Discretization.} In the attached code base we provide an efficient implementation of \adaalgo, the adaptive discretization algorithm for the \CMABCRA domain. Pseudocode for \adaalgo is in \cref{alg:adaptive_discretization}. 

We represent the partition~$\P_t^\effortgeneric$ as a tree with leaf nodes corresponding to \emph{active balls} (i.e. ones which have not yet been split).  Each node in the tree keeps track of $n_t(\reg, \effortgeneric)$, $\mubar_t(\reg, \effortgeneric)$, and $\UCB_t(\reg, \effortgeneric)$.  While the partitioning works for any compact metric space, we implement it in $\pointset = [0,1]^2$ under the infinity norm metric.  With this metric, the high level implementation of the three steps is as follows:
\begin{itemize}
\item \textbf{Selection rule}: In this step we start by ``merging'' all of the trees for each allocation level~$\effortgeneric$ (i.e. $\P_t^\effortgeneric$ for $\effortgeneric \in \effortset$) into a single tree, with estimates $\UCB_t(\reg, \cdot)$ represented as a vector for each $\effortgeneric \in \effortset$ rather than a scalar.  On this merged partition of $\pointset$ we solve \cref{eq:selection} \emph{over the leaves} to avoid modelling the constraint which ensures that each region is selected at a single allocation amount. 
\item \textbf{Update estimates}: Updating the estimates simply updates the stored $n_t(\reg, \effortgeneric), \mubar_t(\reg, \effortgeneric)$, and $\UCB_t(\reg, \effortgeneric)$ for each of the selected regions based on observed data.
\item \textbf{Re-partition the space}:  In order to split a region~$\reg$, we create four new subregions corresponding to splitting the two dimensions in half.  For example, the region $[0,1]^2$ will be decomposed to
\begin{align*}[0, \tfrac{1}{2}] \times [0,\tfrac{1}{2}] \quad [0, \tfrac{1}{2}] \times [\tfrac{1}{2}, 1] \quad  [\tfrac{1}{2}, \tfrac{1}{2}] \times [0, \tfrac{1}{2}] \quad [\tfrac{1}{2}, \tfrac{1}{2}] \times [\tfrac{1}{2}, \tfrac{1}{2}] \ .
\end{align*}
We add on the children to the tree with links to its parent node, and initialize all estimates to that of its parent.
\end{itemize}

Lastly, we comment that in order to implement the $\Full$ and $\ArtReplay$ versions of the adaptive discretization algorithm we pre-processed the entire historical data into a tree structure.  This allowed us to check whether the given action has historical data available by simply checking the corresponding node in the pre-processed historical data tree.
\section{Experiment Details}
\label{sec:experiments_appendix}

We provide additional details about the experimental domains, algorithm implementation, and additional results. The additional results include experiments to evaluate the performance of \ArtReplay on algorithms for combinatorial finite-armed bandit \citep{chen13a} as well as the standard $K$-armed bandit.  For the later we include simulations with Thompson sampling \citep{russo2018tutorial}, and information-directed sampling  \citep{russo2018learning}, which do not satisfy the \iidata property, but still experience empirical improvements when using \ArtReplay against \Full.

\subsection{Domain Details}

\subsubsection{Finite $K$-armed bandit}

For the $K$-armed bandit, we generate mean rewards for each arm $\actiongeneric \in [K]$ uniformly at random. 

\subsubsection{\CMABCRA}

\paragraph{Piecewise-linear}
This synthetic domain has a piecewise-linear reward function to ensure that the greedy approximation solution is optimal, as discussed in \cref{sec:cmab_iid_example}. As a stylized setting, this reward function uses a very simple construction:
\begin{align}
    \mu(\point, \effortgeneric) = \effortgeneric \cdot \left( \frac{p_{1}}{2} + \frac{p_{2}}{2} \right) \ .
\end{align}
We visualize this reward in \cref{fig:reward-pwl}. The optimal reward is at $(1, 1)$.

\paragraph{Quadratic}
The quadratic environment is a synthetic domain with well-behaved polynomial reward function of the form:
    \begin{align}
        \mu(\point, \effortgeneric) = \effortgeneric \left( 1 - (p_{1} - 0.5)^2 + (p_{2} - 0.5)^2 \right)
    \end{align}
which we visualize in \cref{fig:reward-quadratic}. The optimal reward at is achieved at $(0.5, 0.5)$.

\begin{figure}
\centering
\begin{minipage}{.45\textwidth}
  \centering
    \includegraphics[height=.7\linewidth]{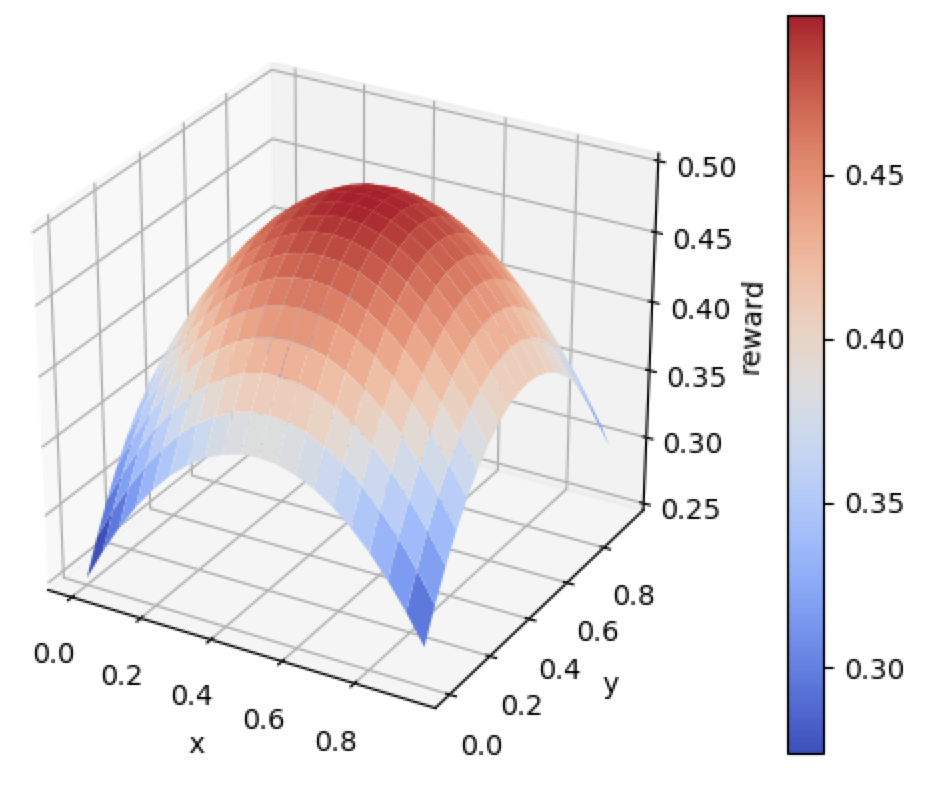}
    \caption{Reward function for the quadratic environment.}
    \label{fig:reward-quadratic}
\end{minipage}%
\qquad
\begin{minipage}{.45\textwidth}
  \centering
  \includegraphics[height=.7\linewidth]{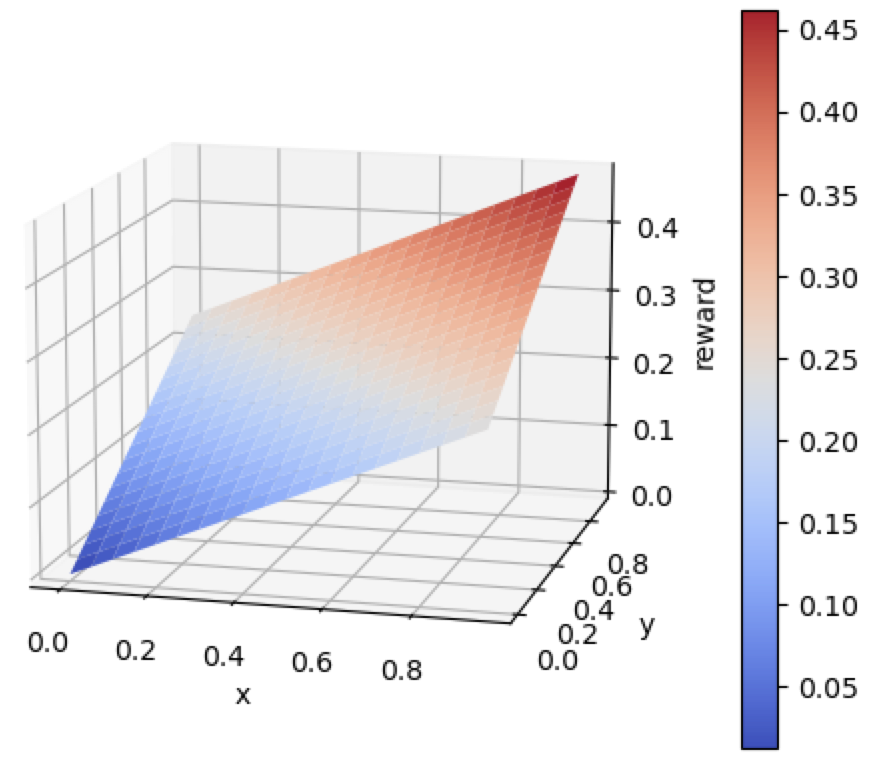}
  \caption{Reward function for the piecewise-linear environment.}
  \label{fig:reward-pwl}
\end{minipage}
\end{figure}



\paragraph{Green Security Domain}
For the green security domain, we wish to predict the probability that poachers place snares throughout a large protected area using ranger patrol observations. 
We use real-world historical patrol data from Murchison Falls National Park. The historical data are continuous-valued GPS coordinates (longitude, latitude) marking trajectories with locations automatically recorded every 30 minutes. Between the years 2015 and 2017, we have 180{,}677 unique GPS waypoints. 

We normalize the space of the park boundary to the range $[0, 1]$ for both dimensions. For each point $\point \in [0, 1]^2$ in this historical data, we compute ``effort'' or allocation by calculating straight-line \emph{trajectories} between the individual waypoints to compute the distance patrolled, allocating to each waypoint one half the sum of the line segments to which it is connected. We then associate with each point a binary label $\{0, 1\}$ representing the observation. Direct observations of poaching are rather rare, so to overcome strong class imbalance for the purposes of these experiments, we augment the set of instances we consider a positive label to include any human-related or wildlife observation. Everything else (e.g., position waypoint) gets a negative label. 

To generate a continuous-action reward function, we build a neural network to learn the reward function across the park and use that as a simulator for reward. The neural network takes three inputs, $\actiongeneric = (\point_1, \point_2, \effortgeneric)$, and outputs a value $[0, 1]$ to indicate probability of an observation. This probability represents $\mu(\actiongeneric)$ for the given point and allocation.







\subsection{Adaptive Discretization}

We offer a visual demonstration of the adaptive discretization process in  \cref{fig:visualize_discretization}, using 10{,}000 real samples of historical patrol observations from Murchison Falls National Park. This discretization is used to iteratively build the dataset tree used to initialize the \Full algorithm with adaptive discretization. 

\begin{figure}
    \centering
    \hspace{3em}
    $\effortgeneric = 0$ \hspace{7.4em} $\effortgeneric = 0.5$  \hspace{7.4em} $\effortgeneric = 1$
    \includegraphics[width=.8\linewidth]{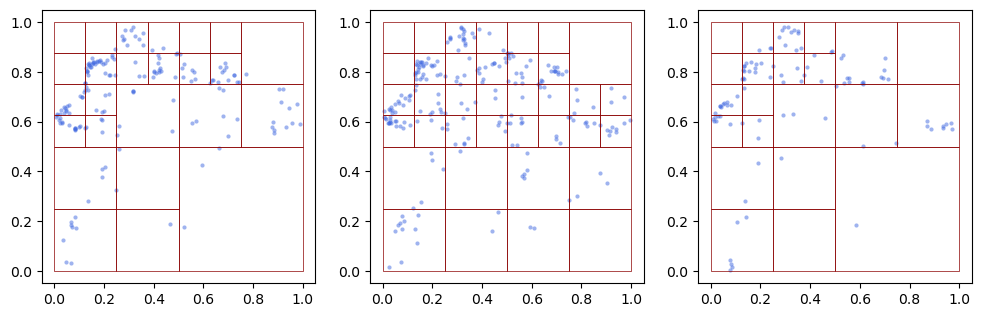}
    
    \vspace{3ex}
    \includegraphics[width=.8\linewidth]{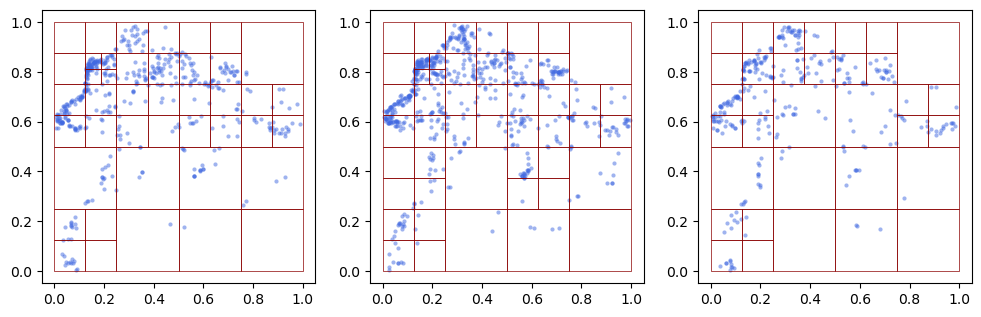}
    
    \vspace{3ex}
    \includegraphics[width=.8\linewidth]{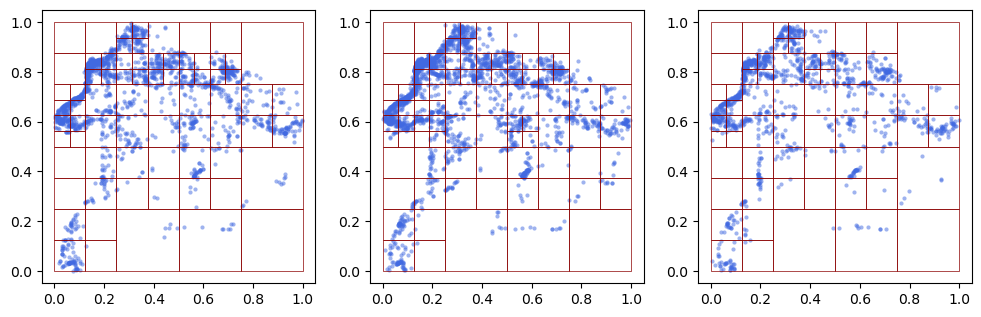}
    
    \vspace{3ex}
    \includegraphics[width=.8\linewidth]{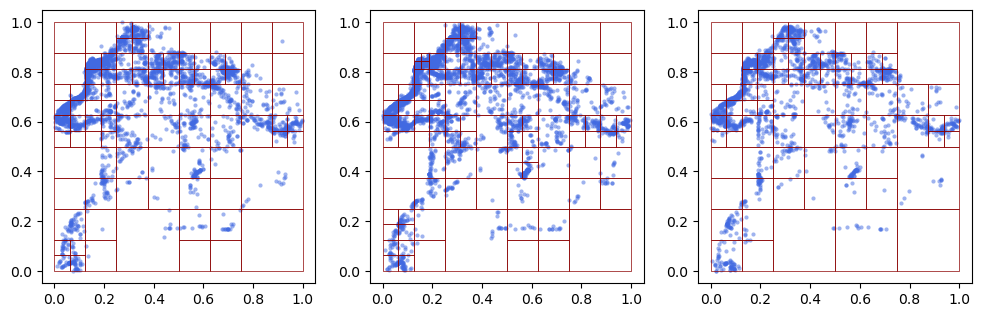}
    \caption{Adaptive discretization in the $\pointset = [0, 1]^2$ space using $10{,}000$ samples of real historical patrol observations from Murchison Falls National Park. Each row depicts the distribution of historical samples and the space partition after 500, 1{,}500, 6{,}000, and 10{,}000 samples are added to the dataset tree. Each column visualizes the dataset tree for each of three levels of effort $\effortgeneric \in \B = \{0, 0.5, 1\}$. As shown, these real-world historical samples exhibit strong \emph{imbalanced data coverage}, leading to significantly fine discretizations in areas with many samples and very coarse discretization in other regions.}
    \label{fig:visualize_discretization}
\end{figure}

\subsection{Additional Experimental Results}

In \cref{fig:k-armed-ts-ids} we evaluate the performance of \ArtReplay compared to \Full and \Ignorant using two multi-armed bandit algorithms that do not have the \iidata property, Thompson Sampling (TS) and Information-Directed Sampling (IDS). Although our theoretical regret guarantees do not apply to Thompson sampling or IDS as base algorithms, these results demonstrate that empirically \ArtReplay still performs remarkably well, matching the performance of \Full with Thompson sampling and avoiding the exploding regret that \Full suffers with IDS. 

We note that with imbalanced data, \Full is converging on a suboptimal action with more historical data.  This is because the IDS algorithm, when warm-started with historical data, maintains a near-`zero entropy' posterior distribution over a sub-optimal action which is over-represented in the historical dataset.  Since the selection procedure takes the action that maximizes expected return divided by posterior variance, the algorithm continuously picks this sub-optimal action at each timestep.

\begin{figure}
    \centering
    \begin{tikzpicture}
\pgfplotsset{
  width=0.4\linewidth,
  height=0.25\linewidth,
  xtick pos=left,
  ytick pos=left,
  tick label style={font=\scriptsize},
  ytick style={draw=none},  
  ymajorgrids=true,
}




\begin{axis}[ 
  at={(.06\linewidth, .19\linewidth)},
  title style={yshift=-.7ex},
  title={\small{\textbf{Thompson sampling}}},
  ylabel={\small{Regret}},
  ylabel near ticks,
  xlabel near ticks,
  xlabel style={shift=-.7ex},
  xtick={0, 250, 500, 750, 1000},
  legend style={
    align=left,
    at={(0.6\linewidth, -.27\linewidth)},
legend columns=3,
    font=\small, 
    draw=none,
    fill=none,
    /tikz/every even column/.append style={column sep=10pt}},
  ymin=0,
  xmin=0, xmax=1000,
]

\addplot [teal!90!white, line width=1pt, mark=none] table [x=t, y=ts_artificial_replay, col sep=comma] {data/regret_ts_ids_k_armed_k10_H100.csv}; \addlegendentry{~\ArtReplay}

\addplot [purple!90!white, line width=1.5pt, dotted, mark=none] table [x=t, y=ts_historical, col sep=comma] {data/regret_ts_ids_k_armed_k10_H100.csv}; \addlegendentry{~\Full}

\addplot [cyan!50!white, line width=1pt, mark=none] table [x=t, y=ts_ignorant, col sep=comma] {data/regret_ts_ids_k_armed_k10_H100.csv}; \addlegendentry{~\Ignorant}

\end{axis}

\begin{axis}[ 
  at={(.45\linewidth, .19\linewidth)},
  title style={yshift=-.7ex},
  title={\small{\textbf{Information-directed sampling}}},
  xlabel near ticks,
  xlabel style={shift=-.7ex},
  xtick={0, 250, 500, 750, 1000},
  ymin=0,
  ymax=22,
  xmin=0, xmax=1000,
]

\addplot [teal!90!white, line width=1pt, mark=none] table [x=t, y=ids_artificial_replay, col sep=comma] {data/regret_ts_ids_k_armed_k10_H100.csv};

\addplot [purple!90!white, line width=1.5pt, dotted, mark=none] table [x=t, y=ids_historical, col sep=comma] {data/regret_ts_ids_k_armed_k10_H100.csv}; 

\addplot [cyan!50!white, line width=1pt, mark=none] table [x=t, y=ids_ignorant, col sep=comma] {data/regret_ts_ids_k_armed_k10_H100.csv}; 

\end{axis}


\begin{axis}[ 
  at={(.06\linewidth, 0\linewidth)},
  xlabel={\footnotesize{Timestep~$t$}},
  ylabel={\small{Regret}},
  xlabel near ticks,
  ylabel near ticks,
  xtick={0, 250, 500, 750, 1000},
  ymin=0,
  xmin=0, xmax=1000,
]

\addplot [teal!90!white, line width=1pt, mark=none] table [x=t, y=ts_artificial_replay, col sep=comma] {data/regret_ts_ids_k_armed_k10_H1000.csv}; 

\addplot [purple!90!white, line width=1.5pt, dotted, mark=none] table [x=t, y=ts_historical, col sep=comma] {data/regret_ts_ids_k_armed_k10_H1000.csv}; 

\addplot [cyan!50!white, line width=1pt, mark=none] table [x=t, y=ts_ignorant, col sep=comma] {data/regret_ts_ids_k_armed_k10_H1000.csv}; 

\end{axis}

\begin{axis}[ 
  at={(.45\linewidth, 0\linewidth)},
  xlabel={\footnotesize{Timestep~$t$}},
  xlabel near ticks,
  xtick={0, 250, 500, 750, 1000},
  ymin=0,
  ymax=22,
  xmin=0, xmax=1000,
]

\addplot [teal!90!white, line width=1pt, mark=none] table [x=t, y=ids_artificial_replay, col sep=comma] {data/regret_ts_ids_k_armed_k10_H1000.csv}; 

\addplot [purple!90!white, line width=1.5pt, dotted, mark=none] table [x=t, y=ids_historical, col sep=comma] {data/regret_ts_ids_k_armed_k10_H1000.csv}; 

\addplot [cyan!50!white, line width=1pt, mark=none] table [x=t, y=ids_ignorant, col sep=comma] {data/regret_ts_ids_k_armed_k10_H1000.csv}; 

\end{axis}


\node[rounded corners, fill=black!10, align=left] at (-.03\linewidth, .19\linewidth) {\small{$H=100$}};
\node[rounded corners, fill=black!10, align=left] at (-.03\linewidth, -.02\linewidth) {\small{$H=1{,}000$}};

\end{tikzpicture}
    \caption{Cumulative regret ($y$-axis; lower is better) across time~$t \in [T]$. \ArtReplay performs competitively across all domain settings, with both Thompson sampling~\citep{russo2018tutorial} (left) and information-directed sampling~\citep{russo2018learning} (right).  In \Full applied to information-directed sampling with $H = 1{,}000$ the algorithm converges on a sub-optimal arm since its posterior variance is low (due to more data), resulting in poor regret performance due to {\em spurious data}.}
    \label{fig:k-armed-ts-ids}
\end{figure}
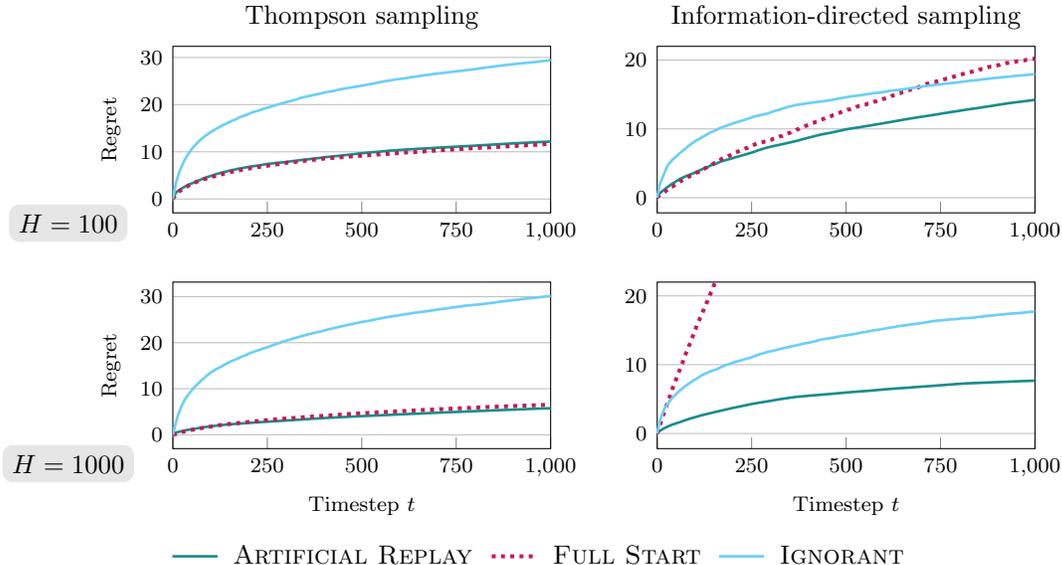

In \cref{fig:k-armed-b3-spurious-data} we consider the combinatorial bandit setting with a set of $K=10$ discrete arms and a budget $B=3$ over $T=1{,}000$ timesteps. This setting is similar to \cref{fig:k-armed-b1-spurious-data} but instead here we consider a combinatorial setting (where multiple arms can be pulled at each timestep) rather than a standard stochastic $K$-armed bandit. Across different values of $H$, \ArtReplay matches the performance of \Full (\cref{fig:k-armed-b3-spurious-data}(a)) despite using an increasing smaller fraction of the historical dataset $\history$ (\cref{fig:k-armed-b3-spurious-data}(b)).
The regret plot in \cref{fig:k-armed-b3-spurious-data}(c) shows that the regret of our method is coupled with that of \Full across time. \cref{fig:k-armed-b3-spurious-data}(d) tracks the number of samples from history~$\history$, with $H=1{,}000$, used over time: \ArtReplay uses 471 historical samples before taking its first online action. The number of historical samples used increases at a decreasing rate and ends with using 740 samples.

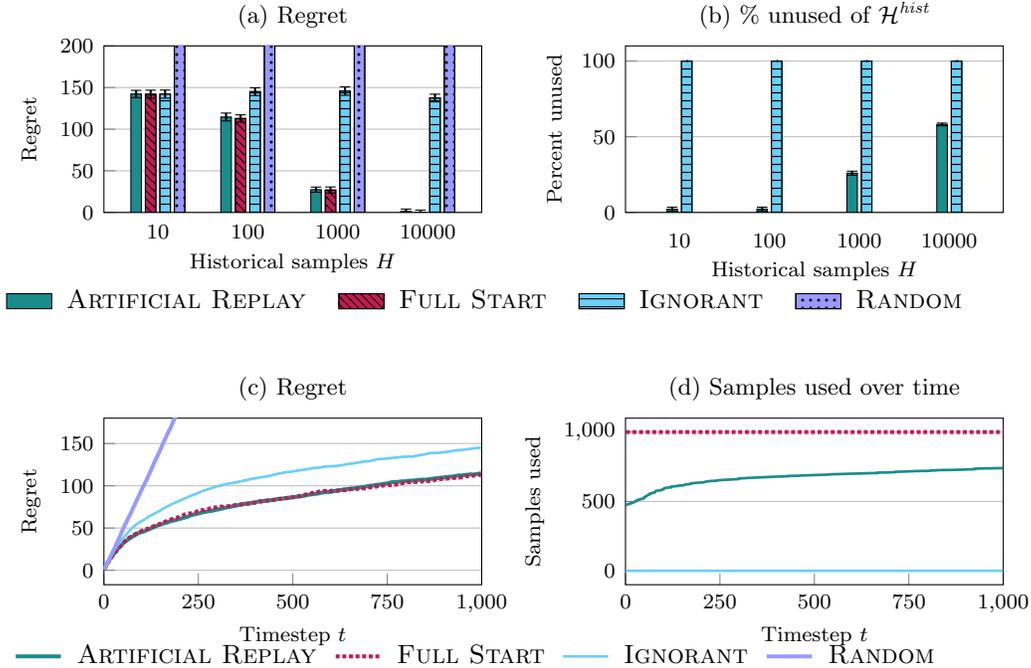
\begin{figure}[!t]
    \centering
    \begin{tikzpicture}
\pgfplotsset{
  width=0.4\linewidth,
  height=0.23\linewidth,
  xtick pos=left,
  ytick pos=left,
  tick label style={font=\scriptsize},
  ymajorgrids=true,
  xlabel shift=-2pt,
  title style={yshift=-.7ex},
}

\begin{axis}[
  at={(.42\linewidth, .3\linewidth)}, 
  table/col sep=comma,
  ybar=1.5pt,
  bar width=.14cm,
  symbolic x coords={10, 100, 1000, 10000},
  xticklabels={10, 100, {1,000}, {10,000}},
  xtick style={draw=none},
  xtick={data},
  xlabel={\footnotesize{Historical samples~$H$}},
  title={\small{\textbf{(b)~\% unused of $\history$}}},
  ylabel style={align=center},
  ylabel={\small{Percent unused}},
  ymin=0, 
  enlarge x limits={0.2},
  error bars/y dir=both, 
  error bars/y explicit,  
  error bars/error bar style={color=black, thick},
]

\addplot+[teal!90!white, draw=black, area legend,
] table [x=H, y=artificial_replay, y error=artificial_replay_sem] {data/k_armed_b3_unused_h_percentage.csv}; 
\addplot+[cyan!50!white, draw=black, area legend,
postaction={pattern=horizontal lines,}
] table [x=H, y=ignorant, y error=ignorant_sem] {data/k_armed_b3_unused_h_percentage.csv}; 
\end{axis}

\begin{axis}[
  at={(0\linewidth, .30\linewidth)}, 
  table/col sep=comma,
  ybar=1.5pt,
  bar width=.14cm,
  symbolic x coords={10, 100, 1000, 10000},
  xticklabels={10, 100, {1,000}, {10,000}},
  xtick style={draw=none},
  x tick label style={yshift=.7ex}, 
  xtick={data},
  xlabel={\footnotesize{Historical samples~$H$}},
  title={\small{\textbf{(a)~Regret}}},
  ylabel style={align=center},
  ylabel={\small{Regret}},
  legend style={
    at={(0.68\linewidth, -.06\linewidth)},
    legend columns=4,
    font=\small,draw=none,fill=none,
    /tikz/every even column/.append style={column sep=10pt}},
  ymin=0, ymax=200,
  enlarge x limits={0.2},
  error bars/y dir=both, 
  error bars/y explicit,  
  error bars/error bar style={color=black, thick},
]

\addplot+[teal!90!white, draw=black, area legend,
] table [x=H, y=artificial_replay, y error=artificial_replay_sem] {data/k_armed_b3_regret.csv}; \addlegendentry{~\ArtReplay~~~~};
\addplot+[purple!90!white, draw=black, area legend,
postaction={pattern=north west lines,}
] table [x=H, y=historical, y error=historical_sem] {data/k_armed_b3_regret.csv}; \addlegendentry{~\Full~~~~};
\addplot+[cyan!50!white, draw=black, area legend,
postaction={pattern=horizontal lines,}
] table [x=H, y=ignorant, y error=ignorant_sem] {data/k_armed_b3_regret.csv}; \addlegendentry{~\Ignorant~~~~};
\addplot+[blue!40!white, draw=black, area legend,
postaction={pattern=dots,}
] table [x=H, y=random, y error=random_sem] {data/k_armed_b3_regret.csv}; \addlegendentry{~\Random~~~~};
\end{axis}

\begin{axis}[ 
  at={(0, 0\textwidth)},
  xlabel={\footnotesize{Timestep~$t$}},
  title={\small{\textbf{(c)~Regret}}},
  xlabel near ticks,
  ylabel={\small{Regret}},
  xtick={0, 250, 500, 750, 1000},
  xticklabels={0, 250, 500, 750, {1,000}},
  legend style={
    align=left,
    at={(0.68\linewidth, -.06\linewidth)},
    legend columns=4,
    font=\small, 
    draw=none,
    fill=none,
    /tikz/every even column/.append style={column sep=10pt}},
  ymin=0,
  ymax=180,
  xmin=0, xmax=1000,
]

\addplot [teal!90!white, line width=1.5pt, mark=none] table [x=t, y=artificial_replay, col sep=comma] {data/k_armed_b3_avg_regret_k10_H100.csv}; \addlegendentry{~\ArtReplay}

\addplot [purple!90!white, line width=1.5pt, densely dotted, mark=none] table [x=t, y=historical, col sep=comma] {data/k_armed_b3_avg_regret_k10_H100.csv}; \addlegendentry{~\Full}

\addplot [cyan!50!white, line width=1pt, mark=none] table [x=t, y=ignorant, col sep=comma] {data/k_armed_b3_avg_regret_k10_H100.csv}; \addlegendentry{~\Ignorant}

\addplot [blue!40!white, line width=1.5pt, mark=none] table [x=t, y=random, col sep=comma] {data/k_armed_b3_avg_regret_k10_H100.csv}; \addlegendentry{~\Random}

\end{axis}

\begin{axis}[ 
  at={(.42\textwidth, 0\textwidth)},
  xlabel={\footnotesize{Timestep~$t$}},
  title={\small{\textbf{(d)~Samples used over time}}},
  xtick={0, 250, 500, 750, 1000},
  xticklabels={0, 250, 500, 750, {1,000}},
  xlabel near ticks,
  ylabel={\small{Samples used}},
  xmin=0, xmax=1000,
]

\addplot [teal!90!white, line width=1pt, mark=none] table [x=t, y=artificial_replay, col sep=comma] {data/k_armed_b3_avg_used_h_k10_H1000.csv}; 

\addplot [purple!90!white, line width=1.5pt, densely dotted, mark=none] table [x=t, y=historical, col sep=comma] {data/k_armed_b3_avg_used_h_k10_H1000.csv}; 

\addplot [cyan!50!white, line width=1pt, mark=none] table [x=t, y=ignorant, col sep=comma] {data/k_armed_b3_avg_used_h_k10_H1000.csv}; 

\end{axis}

\end{tikzpicture}
    \caption{We consider a combinatorial bandit setting with finite actions: $K=10$ arms, $B=3$ budget, and horizon $T=1{,}000$. Increasing the number of historical samples~$H$ leads \Full to use unnecessary data, particularly as $H$ gets very large. \ArtReplay achieves equal performance in terms of regret (plot~a) while using less than half the historical data (plot~b). In (plot~c) we see that with $H=1{,}000$ historical samples, \ArtReplay uses 471 historical samples before taking its first online action. The number of historical samples used increases at a decreasing rate, using 740 total samples by the horizon~$T$.
    }
    \label{fig:k-armed-b3-spurious-data}
\end{figure}

\subsection{Experiment Execution} 

Each experiment was run with $60$ iterations where the relevant plots are taking the mean of the related quantities.  All randomness is dictated by a seed set at the start of each simulation for verifying results.  The experiments were conducted on a personal laptop with a 2.4 GHz Quad-Core Intel Core i5 processor and 16~GB of RAM. 

\section{Omitted Proofs}
\label{app:proofs}

\subsection{\cref{sec:iidata_coupling} Proofs}
\label{app:proof_coupling}

\begin{rproofof}{\cref{thm:coupling}}
We start off by showing that $\pi^{\ArtReplay(\basealgo)}_t \eqdist \pi^{\Full(\basealgo)}$ using the reward stack model for a stochastic bandit instance introduced in~\citet{lattimore2020bandit}.  Due to the fact that the observed rewards are independent (both across actions but also across timesteps), consider a sample path where $(\obsrewardt{\actiongeneric, t})_{\actiongeneric \in \actionset, t \in [T]}$ are pre-sampled according to $\rewarddist(\actiongeneric)$.  Upon pulling arm~$\actiongeneric$ in timestep~$t$, the algorithm is given feedback $R_{a, t}$.

It is important to note that the resulting probability space generated in the reward stack model is identical in distribution to any sequence of histories observed by running a particular algorithm.
More specifically, it preserves the following two properties:
\begin{itemize}
    \item[(a)] The conditional distribution of the action $\actiont{t}$ given the sequence $(\actiont{1}, \obsrewardt{\actiont{1}, 1}), \ldots, (\actiont{t-1}, \obsrewardt{\actiont{t-1}, t-1})$ is $\pi_t(\cdot \mid \H_t)$ almost surely.
    \item[(b)] The conditional distribution of the reward $\obsrewardt{t}$ is $\rewarddist(\actiont{t})$ almost surely.
\end{itemize}

Based on this reward stack model, we show by induction on~$t$ that  $\pi^{\ArtReplay(\basealgo)}_t = \pi_t^{\Full(\basealgo)}$.  Since this is true on an independent sample path, it results in a probabilistic coupling between the two algorithms, implying that the chosen online actions and collected rewards have the same distribution.

\vspace{1ex}
\noindent \textbf{Base Case}: $t = 1$.

By definition of $\pi^{\Full(\basealgo)}$ we know that
\begin{align*}
    \pi^{\Full(\basealgo)}_1 = \basealgo(\history) \ .
\end{align*}
However, consider $\pi^{\ArtReplay(\basealgo)}$.  The \ArtReplay meta-algorithm will keep selecting actions until it creates a dataset $\Hon_1 \subset \history$ such that $\basealgo(\Hon_1)$ has no more unused samples in $\history$, i.e. $\history \setminus \Hon_1 \cap \{\actiongeneric \sim_{\Hon_1} \basealgo(\Hon_1)\} = \emptyset$.  Denoting $\actiont{1} = \basealgo(\Hon_1)$, the unused samples $\history \setminus \Hon_1$ contains no data on $\actiont{1}$.  As a result, by the independence of irrelevant data property for $\basealgo$ we have that $\basealgo(\Hon_1) = \basealgo(\history)$ and so $\pi^{\Full(\basealgo)}_1 = \pi^{\ArtReplay(\basealgo)}_1$.  Note that this shows that the observed online data for the algorithms as denoted by $\H_2 = \{\actiont{1}, \obsrewardt{\actiont{1}, 1}\}$ are also identical (due to the reward stack model).

\vspace{1ex}
\noindent \textbf{Step Case}: $t-1 \rightarrow t$.

Since we know that $\pi^{\Full(\basealgo)}_\tau = \pi^{\ArtReplay(\basealgo)}_\tau$ for $\tau < t$, both algorithms have access to the same set of observed online data $\H_t$. By definition of $\pi^{\Full(\basealgo)}$:
\begin{align*}
    \pi^{\Full(\basealgo)}_t = \basealgo(\history \cup \H_t) \ .
\end{align*}

However, the \ArtReplay algorithm continues to use offline samples until it generates a subset $\Hon_t \subset \history$ such that $\basealgo(\Hon_t \cup \H_t)$ has no further samples in $\history$, i.e. \[(\history \setminus \Hon_t) \cap \{\actiongeneric \in \actionset \mid \actiongeneric \rel_{\Hon_t \cup \H_t} \basealgo(\Hon_t \cup \H_t)\} = \emptyset.\]  Hence, by the independence of irrelevant data property again:
\[
    \basealgo(\history \cup \H_t) = \basealgo(\Hon_t \cup \H_t) \ ,
\]
and so $\pi^{\Full(\basealgo)}_t = \pi^{\ArtReplay(\basealgo)}_t$.  Again we additionally have that $\H_{t+1} = \H_t \cup \{(\actiont{t}, R_{\actiont{t}, t})\}$ are identical for both algorithms.

Together this shows that $\pi^{\Full(\basealgo)} \eqdist \pi^{\ArtReplay(\basealgo)}$.  Lastly we note that the definition of regret is $\regret(T, \pi, \history) = T \cdot \OPT - \sum_{t=1}^T \mu(\actiont{t})$ where $\actiont{t}$ is sampled from $\pi$.  Hence the policy-based coupling implies that
$\regret(T, \pi^{\ArtReplay(\basealgo)}, \history) \eqdist \regret(T, \pi^{\Full(\basealgo)}, \history)$ as well.
\end{rproofof}

\begin{rproofof}{\cref{thm:generic_regret_gain}}
By definition, we have that the regret of $\pi^{\ArtReplay(\basealgo)}$ is given by:
\begin{align*}
    \regret(T, \pi^{\ArtReplay(\basealgo)}) = \sum_{t=1}^T \OPT - \mu(\Pi(\H_t \cup \Hon_t)),
\end{align*}
where $\H_t$ corresponds to the online data observed by the algorithm at the start of round $t$, and $\Hon_t \subset \history$ the set of used data points by \ArtReplay by the time the final action is selected.  However, within each round $t$, \ArtReplay will propose a sequence of selected actions $\tilde{A}_t^\tau$ where $\tau$ indexes over the number of selected actions considered before finally picking an action which has no related samples in the historical dataset.  Denote by $\Hon_t(\tau) \subset \Hon_t$ as the set of historical data points considered by \ArtReplay upon selecting the proposed action $\tilde{A}_t^\tau$.  We note that $\Hon_t(\tau)$ are nested, and that if $I_t$ is the number of proposed actions selected by the algorithm in round $t$, $\Hon_t(I_t+1) = \Hon_t$.  Hence we have that:
\begin{align*}
        & \regret(T, \pi^{\ArtReplay(\basealgo)}) \\
        & = \sum_{t=1}^T \OPT - \mu(\Pi(\H_t \cup \Hon_t)) \\
        & = \sum_{t=1}^T \sum_{\tau \in [I_t + 1]} \OPT - \mu(\Pi(\H_t \cup \Hon_t(\tau))) - \sum_{t=1}^T \sum_{\tau \in [I_t]} \OPT - \mu(\Pi(\H_t \cup \Hon_t(\tau))) \\
        & = \regret(T + |\Hon_T|, \pi^{\Ignorant(\basealgo)}) - \sum_{t=1}^T \sum_{\tau \in [I_t]} \OPT - \mu(\Pi(\H_t \cup \Hon_t(\tau))),
\end{align*}
where in the final equality we used the fact that the first term is equal to the regret of the base algorithm $\basealgo$ on a sequence of $T + |\Hon_T|$.  We note that this explicitly uses the fact that the sequence of sets for which $\basealgo$ is making the decisions are properly nested and that one data point is added within each set.  Lastly, we use the definition of $\Delta$ to upper bound the right hand side and use the definition of $\Hon_T(\dagger)$ to obtain the final result as follows.
\begin{align*}
        & \regret(T, \pi^{\ArtReplay(\basealgo)}) \\
        & = \regret(T + |\Hon_T|, \pi^{\Ignorant(\basealgo)}) - \sum_{t=1}^T \sum_{\tau \in [I_t]} \OPT - \mu(\Pi(\H_t \cup \Hon_t(\tau))) \\
        & \leq \regret(T + |\Hon_T|, \pi^{\Ignorant(\basealgo)}) - \Delta_{\min} |\Hon_T(\dagger)|.
\end{align*}
The final result follows by taking an expectation of both sides.
\end{rproofof}

\subsection{\cref{sec:ar_case_study_mon_ucb}, \cref{sec:mon_ucb_benefits}, \cref{sec:ucb_algorithms_appendix} Proofs for \MonUCB}
\label{sec:k_armed_proofs}
\begin{rproofof}{\cref{thm:mon_ucb_iid_psi}}
Suppose that the base algorithm~$\basealgo$ is \MonUCB for any convex function $\psi$, and let $\H$ be an arbitrary dataset.  Using the dataset, $\MonUCB$ will construct upper confidence bound values $\UCB(\actiongeneric)$ for each action $\actiongeneric \in [K]$.  The resulting policy is to pick the action $\basealgo(\H) = \argmax_{\actiongeneric \in [K]} \UCB(\actiongeneric)$.  Let $\actiont{\H}$ be the action which maximizes the $\UCB(\actiongeneric)$ value.  

Additionally, let $\H^\prime$ be an arbitrary dataset containing observations from actions other than $\actiont{\H}$.  Based on the enforced monotonicty of the indices from \MonUCB, any $\actiongeneric \in [K]$ with $\actiongeneric \neq \actiont{\H}$ will have its constructed $\UCB(\actiongeneric)$ no larger than its original one constructed with only using the dataset~$\H$.  Moreover, $\UCB(\actiont{\H})$ will be unchanged since the additional data~$\H^\prime$ does not contain any information on $\actiont{\H}$.  As a result, the policy will take $\basealgo(\H \cup \H^\prime) = \actiont{\H}$ since it will still maximize the $\UCB(\actiongeneric)$ index.

Next we provide a regret analysis for $\Ignorant(\MonUCB)$.  We assume without loss of generality that there is a unique action $\aopt$ which maximizes $\mu(\actiongeneric)$.  We let $\Delta(\actiongeneric) = \mu(\aopt) - \mu(\actiongeneric)$ be the sub-optimality gap for any other action $\actiongeneric$.  To show the regret bound we follow the standard regret decomposition, which notes that $\Exp{\regret(T, \pi, \history)} = \sum_{a} \Delta(a) \Exp{n_T(a)}$.  To this end, we start off with the following Lemma, giving a bound on the expected number of pulls of $\MonUCB$ in the ``ignorant'' setting (i.e., without incorporating any historical data).
\begin{lemma}
\label{lem:ig_pull_count}
The expected number of pulls for any sub-optimal action~$\actiongeneric$ of \MonUCB with convex function $\psi$ satisfies
\[
\Exp{n_T(\actiongeneric)} \leq \frac{2\log(2TK)}{\psiopt(\Delta(a) / 2)} + 1 \ .
\]
\end{lemma}
\begin{rproofof}{\cref{lem:ig_pull_count}}
Denote by $S_{\actiongeneric, \tau}$ to be the empirical sum of $\tau$~samples from action~$\actiongeneric$.  Note that via an application of Markov's inequality along with \cref{ass:convex_psi}:
\[
    \Pr\left(\left|\mu(\actiongeneric) - \frac{S_{\actiongeneric, \tau}}{\tau}\right| \geq (\psiopt)^{-1}\left(\frac{2\log(2TK)}{\tau}\right)\right) \leq \frac{1}{2T^2K^2} \ .
\]
A straightforward union bound with this fact shows that the following event occurs with probability at least $1 - 1/T$:
\[
    \mathcal{E} = \left\{\forall \actiongeneric \in [K], 1 \leq k \leq T, \left|\mu(\actiongeneric) - S_{\actiongeneric,k}/k\right| \leq (\psiopt)^{-1}\left(\frac{2\log(2TK)}{\tau}\right) \right\} \ .
\]
Now consider an arbitrary action $\actiongeneric \neq \aopt$.  We start by showing that on the event $\mathcal{E}$ that $n_t(\actiongeneric) \leq 4\log(T) / \Delta(\actiongeneric)^2$.  If action~$\actiongeneric$ was taken over $\aopt$ at some timestep $t$ then:
\[
    \UCB_t(\actiongeneric) > \UCB_t(\aopt) \ .
\]
However, using the reward stack model and the definition of $\UCB_t(\actiongeneric)$ we know that
\begin{align*}
\UCB_t(\actiongeneric) & = \min_{\tau \leq t} \frac{S_{\actiongeneric, n_\tau(\actiongeneric)}}{n_\tau(\actiongeneric)} + (\psiopt)^{-1}\left(\frac{2\log(2TK)}{n_t(a)}\right) \quad \text{by definition of monotone UCB}\\
& \leq \frac{S_{\actiongeneric, n_t(\actiongeneric)}}{n_t(\actiongeneric)} + (\psiopt)^{-1}\left(\frac{2\log(2TK)}{n_t(a)}\right).
\end{align*}
Moreover, under the good event $\mathcal{E}$ we know that $\mu(\aopt) \leq \UCB_t(\aopt)$ and that \[
\mu(\actiongeneric) \geq \frac{S_{\actiongeneric, n_t(\actiongeneric)}}{n_t(\actiongeneric)} - (\psiopt)^{-1}\left(\frac{2\log(2TK)}{n_\tau(a)}\right).\]
Combining this together gives
\begin{align*}
    \mu(\aopt) \leq \mu(\actiongeneric) + 2 (\psiopt)^{-1}\left(\frac{2\log(2TK)}{n_t(\actiongeneric)}\right) \ .
\end{align*}
Rearranging this inequality gives that $n_t(\actiongeneric) \leq \frac{2\log(2TK)}{\psiopt(\Delta(a) / 2)}$.  Lastly, the final bound comes from the law of total probability and the bound on $\Pr(\mathcal{E})$.
\end{rproofof}

Lastly, \cref{lem:ig_pull_count} can be used to obtain a bound on regret via: 
\[
\Exp{\regret(T, \pi^{\MonUCB}, \history)} \leq O\left( \sum_{a} \frac{2\Delta(a) \log(2TK)}{\psiopt(\Delta(a)/2)}\right)
\] using the previous regret decomposition, recovering the regret bound of standard $\psi$ UCB.
\end{rproofof}

In order to show \cref{thm:comp_gains_psi,thm:reg_gains_psi} we start by showing a lemma similar to \cref{lem:ig_pull_count} but for $\Full(\MonUCB)$.

\begin{lemma}
\label{lem:pull_count}
Let $H_\actiongeneric$ be the number of data points in $\history$ for an action $\actiongeneric \in [K]$.  The expected number of pulls for any sub-optimal action~$\actiongeneric$ of \Full(\MonUCB) satisfies
\[
\Exp{n_T(\actiongeneric) \mid H_a} \leq 1 + \max\left\{0, \frac{2\log(2TK)}{\psiopt(\Delta(\actiongeneric)/2)} - H_\actiongeneric \right\} \ .
\]
\end{lemma}
\begin{rproofof}{\cref{lem:pull_count}}
We replicate the proof of \cref{lem:ig_pull_count} but additionally add $H_\actiongeneric$~samples to the estimates of each action $\actiongeneric \in [K]$.  Denote by $S_{\actiongeneric, \tau}$ the empirical sum of $\tau + H_\actiongeneric$ samples from action~$\actiongeneric$.  Note that via an application of Hoeffding's inequality:
\[
    \Pr\left(\left|\mu(\actiongeneric) - \frac{S_{\actiongeneric, \tau}}{\tau+H_\actiongeneric}\right| \geq (\psiopt)^{-1}\left(\frac{2\log(2TK)}{\tau+H_a}\right)\right) \leq \frac{1}{2T^2K^2} \ .
\]
A straightforward union bound with this fact shows that the following event occurs with probability at least $1 - 1/T$:
\[
    \mathcal{E} = \left\{\forall \actiongeneric \in [K], 1 \leq k \leq T, \left|\mu(\actiongeneric) - S_{a,k}/(k+H_\actiongeneric)\right| \leq (\psiopt)^{-1}\left(\frac{2\log(2TK)}{\tau+H_a}\right) \right\} \ .
\]
Now consider an arbitrary action $\actiongeneric \neq \aopt$.  If action~$\actiongeneric$ was taken over $\aopt$ at some timestep $t$ then:
\[
    \UCB_t(\actiongeneric) > \UCB_t(\aopt) \ .
\]
By definition of the $\UCB_t(\actiongeneric)$ term we know 
\begin{align*}
\UCB_t(\actiongeneric) & = \min_{\tau \leq t} \frac{S_{\actiongeneric, n_\tau(\actiongeneric)}}{n_\tau(\actiongeneric)+H_\actiongeneric} + (\psiopt)^{-1}\left(\frac{2\log(2TK)}{n_\tau(\actiongeneric) +H_{\actiongeneric}}\right) \tag{by definition of monotone UCB}\\
& \leq \frac{S_{\actiongeneric, n_t(\actiongeneric)}}{n_t(\actiongeneric)+H_\actiongeneric} + (\psiopt)^{-1}\left(\frac{2\log(2TK)}{n_t(\actiongeneric) +H_{\actiongeneric}}\right) \ .
\end{align*}
Moreover, under the good event $\mathcal{E}$ we know that $\mu(\aopt) \leq \UCB_t(\aopt)$ and that \[
\mu(\actiongeneric) \geq \frac{S_{\actiongeneric, n_t(\actiongeneric)}}{n_t(\actiongeneric)+H_\actiongeneric} - (\psiopt)^{-1}\left(\frac{2\log(2TK)}{n_t(\actiongeneric) +H_{\actiongeneric}}\right) \ .\]
Combining this together gives
\begin{align*}
    \mu(\aopt) \leq \mu(\actiongeneric) + 2 (\psiopt)^{-1}\left(\frac{2\log(2TK)}{n_t(\actiongeneric) +H_{\actiongeneric}}\right) \ .
\end{align*}
Rearranging this inequality gives that $n_t(\actiongeneric) \leq \frac{2\log(2TK)}{\psiopt(\Delta(\actiongeneric)/2)} - H_\actiongeneric$.  Lastly, the final bound comes from the law of total probability and the bound on $\Pr(\mathcal{E})$.
\end{rproofof}
 
Using the previous two lemmas we are able to show \cref{thm:comp_gains_psi,thm:reg_gains_psi}.

\begin{rproofof}{\cref{thm:comp_gains_psi}}
Without loss of generality we consider a setting with $K = 2$.  Let $\aopt$ denote the optimal arm and~$\actiongeneric$ the other arm.  Consider the historical dataset as follows: $\history = (\actiongeneric, \rewardhistt{j})_{j \in [H]}$ where each $\rewardhistt{j} \sim \rewarddist(\actiongeneric)$.  By definition of $\Full(\MonUCB)$ we know that the time complexity of the algorithm is at least $T + H$ since the algorithm will process the entire historical dataset.

In contrast, $\ArtReplay(\MonUCB)$ will stop playing action~$\actiongeneric$ after $O(\log(T) / \psiopt(\Delta(\actiongeneric)))$ timesteps via \cref{lem:ig_pull_count}.  Hence the time complexity of $\ArtReplay(\MonUCB)$ can be upper bound by $O(T + \log(T))$.
\end{rproofof}

\begin{rproofof}{\cref{thm:reg_gains_psi}}
First via \cref{thm:coupling} we note that in order to analyze $\pi^{\ArtReplay(\MonUCB)}$ it suffices to consider $\pi^{\Full(\MonUCB)}$.  Using \cref{lem:pull_count} and a standard regret decomposition we get that:
\begin{align*}
    \Exp{\regret(T, \pi^{\Full(\MonUCB)}, \history) \mid H_a} & = \sum_{\actiongeneric \neq \aopt} \Delta(\actiongeneric) \Exp{n_T(\actiongeneric) \mid H_a} \quad \text{ (by regret decomposition)}\\
    & \leq \sum_{\actiongeneric \neq \aopt} \Delta(\actiongeneric) O\left(\max\left\{0, \frac{\log(TK)}{\psiopt(\Delta(\actiongeneric))} - H_\actiongeneric\right\} \right) \quad \text{ (by \cref{lem:pull_count})} \\
    & = \sum_{\actiongeneric \neq \aopt} O\left(\max\left\{0, \frac{\Delta(a)\log(TK)}{\psiopt(\Delta(\actiongeneric))} - H_\actiongeneric \Delta(\actiongeneric) \right\} \right).
\end{align*}
\end{rproofof}

\subsection{\cref{sec:app_algorithms_metric_bandits} Proofs}
\label{sec:proofs_metric_bandits}
\begin{rproofof}{\cref{thm:iid_metric_bandits}} We first consider the fixed discretization algorithm.  Denote by $\mubar(\H, \reg)$ as the average reward for samples accumulated in a region $\reg$ over dataset $\H$, and $n(\H, \reg)$ the number of times points in $\reg$ have been selected based on dataset $\H$.  Lastly we let $\UCB(\H, \reg)$ be the monotone UCB value for a region $\reg$ based on dataset $\H$ (which is ensured to be monotone by taking the minimum value).  By definition of the fixed discretization algorithm, we have that
\[
\basealgo(\H) = \argmax_{\reg \in \P_1} \UCB(\H, \reg).
\]
Let $\reg^*$ denote the region that maximizes the $\UCB$ value and $\actiongeneric^* \in \reg^*$ arbitrary.  To check for the \iidata property, let $\H'$ be any other dataset containing samples for actions $\actiongeneric \in \actionset$ such that $d(\actiongeneric, \actiongeneric^*) > \gamma$.  Then by construction of the discretization, it must follow that $\actiongeneric$ is contained in regions other than $\reg^*$.
Thus for regions $\reg$ other than $\reg^*$ we have $\UCB(\H \cup \H', \reg) \leq \UCB(\H, \reg)$, and $\UCB(\H \cup \H', \reg^*) = \UCB(\H, \reg^*)$.  Thus, $\reg^*$ remains as the region which maximizes $\UCB(\H \cup \H', \reg)$.

The proof for the adaptive discretization follows similarly.  Let $\H'$ be any other dataset containing samples for actions $\actiongeneric$ in regions other than $\reg$ based on the current partition $\P(\H)$ (note that while in the algorithm description we indexed the partition by $t$, it is solely a function of the observed data thus far).  We emphasize that $\P(\H \cup \H') \subset \P(\H)$.  Moreover, $\reg^* \in \P(\H \cup \H')$ since $\H'$ contains data for actions other than those contained in $\reg^*$.  Hence we have that:
\[
\UCB(\H \cup \H', \reg) \leq \UCB(\H \cup \H', \reg)
\]
for any $\reg \in \P(\H \cup \H')$ with $\reg \neq \reg^*$ since the UCB values are enforced to be monotone decreasing, and child regions inherent the values from their parents.  Moreover, $\UCB(\H \cup \H', \reg^*) = \UCB(\H, \reg^*)$.  Thus, $\reg^*$ again remains as the region which maximizes $\UCB(\H \cup \H', \reg)$ over all $\reg \in \P(\H \cup \H')$.
Here we see the necessity that the information relation is history-dependent, since the discretization changes as a function of the observed data.
\end{rproofof}

\subsection{\cref{sec:app_algorithms_cmab_cra} Proofs}
\label{sec:proofs_cmab_bandits}

\begin{rproofof}{\cref{thm:iid_example}}
We first consider the fixed discretization algorithm.  Using the dataset, the algorithm constructs UCB values $\UCB(\reg, \beta)$ for each $\reg$ and $\beta \in \effortset$ in the discretization $\P_1$ of $\pointset$.  The resulting policy is to pick the action which solves an optimization problem of the form:
\begin{align}
    \label{eq:opt_value_appendix}
    \max_{z} & \sum_{\beta \in \effortset} \sum_{\reg \in \P_1^\beta} \UCB(\reg, \beta) z(\reg, \beta) \\
    \text{s.t. } & \sum_{\beta \in \effortset} \sum_{\reg \in \P_1^\beta} \beta z(\reg, \beta) \leq B \nonumber \\
    & \sum_{\beta \in \effortset} \sum_{\reg \in \P_1^\beta} z(\reg, \beta) \leq N \nonumber \\ 
    & z(\reg, \beta) + \sum_{\beta' \neq \beta} \sum_{\tilde{\reg} \in \P_1^{\beta'}, \reg \subset \tilde{\reg}} z(\tilde{\reg}, \beta') \leq 1 \quad \forall \beta, \reg \in \P_1^\beta \nonumber
\end{align}
where $\UCB(\reg, \effortgeneric)$ is constructed based on~$\H$. Let $\actiont{\H} = ((\reg_1, \effortgeneric_1), \ldots, (\reg_N, \effortgeneric_N))$ be the resulting combinatorial action from solving the optimization problem.  Note here again we abuse notation slightly and assume the algorithm plays a fixed point $\point \in \reg$ for each region $\reg_i$.  By definition of the information relation, let $\H'$ be an arbitrary dataset containing observations of the form $(\point, \effortgeneric)$ for any $(\point, \effortgeneric)$ such that for each $i$ either $(i)$ $\point \not\in \reg_i$ or $(ii)$ $\effortgeneric \neq \effortgeneric_i$.  After updating the UCB estimates with $\H'$ we have that $\UCB(\reg, \beta)$ decreases for any $\reg, \beta$ with $\reg \neq \reg_i$ or $\beta \neq \effortgeneric_i$.  Hence, since the algorithm is solving the expression using a greedy selection rule, we know the resulting action is the same since the selected regions UCB estimates were unchanged and the rest decreased.

The proof for the adaptive discretization algorithm follows similarly using the idea from the proof of \cref{thm:iid_metric_bandits} in \cref{sec:proofs_metric_bandits}, where we again exploit the fact that \cref{eq:opt_value_appendix} is solved greedily.
\end{rproofof}

\end{APPENDICES}

\end{document}